\documentclass{article}

\usepackage{bbm}
\PassOptionsToPackage{authoryear,round}{natbib}

  \usepackage[final]{neurips_2019}

\usepackage[utf8]{inputenc} 
\usepackage[T1]{fontenc}    
\usepackage{hyperref}       
\usepackage{url}            
\usepackage{booktabs}       
\usepackage{amsfonts}       
\usepackage{nicefrac}       
\usepackage{microtype}      

\usepackage{amsmath,amsthm,amssymb}
\usepackage{thmtools}
\usepackage{algorithm}
\usepackage[noend]{algpseudocode}
\usepackage{enumitem}
\usepackage{graphicx}
\usepackage{wrapfig}
\usepackage{subcaption}

\newtheorem{assumption}{Assumption}

\title{Exploration via Hindsight Goal Generation}

\author{
    Zhizhou Ren\thanks{Work done while Zhizhou and Kefan were visiting students at UIUC.} , Kefan Dong\footnotemark[2] \\
    Institute for Interdisciplinary Information Sciences, Tsinghua University\\
    Department of Computer Science, University of Illinois at Urbana-Champaign \\
    \texttt{\{rzz16, dkf16\}@mails.tsinghua.edu.cn} \\
    \And
    Yuan Zhou \\
    Department of Industrial and Enterprise Systems Engineering \\
    University of Illinois at Urbana-Champaign \\
    \texttt{yuanz@illinois.edu} \\
    \And
    Qiang Liu \\
    Department of Computer Science \\
    University of Texas at Austin \\
    \texttt{lqiang@cs.utexas.edu} \\
    \And
    Jian Peng \\
    Department of Computer Science \\
    University of Illinois at Urbana-Champaign \\
    \texttt{jianpeng@illinois.edu} \\
}

\begin{document}

\maketitle

\begin{abstract}
    Goal-oriented reinforcement learning has recently been a practical framework for robotic manipulation tasks, in which an agent is required to reach a certain goal defined by a function on the state space. However, the sparsity of such reward definition makes traditional reinforcement learning algorithms very inefficient. Hindsight Experience Replay (HER), a recent advance, has greatly improved sample efficiency and practical applicability for such problems. It exploits previous replays by constructing imaginary goals in a simple heuristic way, acting like an implicit curriculum to alleviate the challenge of sparse reward signal. In this paper, we introduce Hindsight Goal Generation (HGG), a novel algorithmic framework that generates valuable hindsight goals which are easy for an agent to achieve in the short term and are also potential for guiding the agent to reach the actual goal in the long term.
    We have extensively evaluated our goal generation algorithm on a number of robotic manipulation tasks and demonstrated substantially improvement over the original HER in terms of sample efficiency.
\end{abstract}

\section{Introduction}

Recent advances in deep reinforcement learning (RL), including policy gradient methods \citep{Schulman2015, Schulman2017ProximalPO} and Q-learning \citep{mnih2015human}, have demonstrated a large number of successful applications in solving hard sequential decision problems, including robotics \citep{Levine2016}, games \citep{Silver16, mnih2015human}, and recommendation systems \citep{Karatzoglou2013}, among others. To train a well-behaved policy, deep reinforcement learning algorithms use neural networks as functional approximators to learn a state-action value function or a policy distribution to optimize a long-term expected return. The convergence of the training process, particularly in Q-learning, is heavily dependent on the temporal pattern of the reward function \citep{Szepesvari1998}. For example, if only a non-zero reward/return is provided at the end of an rollout of a trajectory with length $L$, while no rewards are observed before the $L$-th time step, the Bellman updates of the Q-function would become very inefficient, requiring at least $L$ steps to propagate the final return to the Q-function of all earlier state-action pairs. Such sparse or episodic reward signals are ubiquitous in many real-world problems, including complex games and robotic manipulation tasks \citep{andrychowicz2017hindsight}. Therefore, despite its notable success, the application of RL is still quite limited to real-world problems, where the reward functions can be sparse and very hard to engineer \citep{Ng1999}. In practice, human experts need to design reward functions which would reflect the task needed to be solved and also be carefully shaped in a dense way for the optimization in RL algorithms to ensure good performance. However, the design of such dense reward functions is non-trivial in most real-world problems with sparse rewards. For example, in goal-oriented robotics tasks, an agent is required to reach some state satisfying predefined conditions or within a state set of interest. Many previous efforts have shown that the sparse indicator rewards, instead of the engineered dense rewards, often provide better practical performance when trained with deep Q-learning and policy optimization algorithms \citep{ andrychowicz2017hindsight}. In this paper, we will focus on improving training and exploration for goal-oriented RL problems. 

A notable advance is called {\it Hindsight Experience Replay (HER)} \citep{andrychowicz2017hindsight}, which greatly improves the practical success of off-policy deep Q-learning for goal-oriented RL problems, including several difficult robotic manipulation tasks. The key idea of HER is to revisit previous states in the experience replay and construct a number of achieved hindsight goals based on these visited intermediate states. Then the hindsight goals and the related trajectories are used to train an universal value function parameterized by a goal input by algorithms such as deep deterministic policy gradient (DDPG, \citet{lillicrap2016continuous}). A good way to think of the success of HER is to view HER as an implicit curriculum which first learns with the intermediate goals that are easy to achieve using current value function and then later with the more difficult goals that are closer to the final goal. A notable difference between HER and curriculum learning is that HER does not require an explicit distribution of the initial environment states, which appears to be more applicable to many real problems. 

In this paper, we study the problem of automatically generating valuable hindsight goals which are more effective for exploration. Different from the random curriculum heuristics used in the original HER, where a goal is drawn as an achieved state in a random trajectory, we propose a new approach that finds intermediate goals that are easy to achieve in the short term and also would likely lead to reach the final goal in the long term. To do so, we first approximate the value function of the actual goal distribution by a lower bound that decomposes into two terms, a value function based on a hindsight goal distribution and the Wasserstein distance between the two distributions.  Then, we introduce an efficient discrete Wasserstein Barycenter solver to generate a set of hindsight goals that optimizes the lower bound. Finally, such goals are used for exploration.

In the experiments, we evaluate our Hindsight Goal Generation approach on a broad set of robotic manipulation tasks. By incorporating the hindsight goals, a significant improvement on sample efficiency is demonstrated over DDPG+HER. Ablation studies show that our exploration strategy is robust across a wide set of hyper-parameters.

\section{Background}


\textbf{Reinforcement Learning}
    The goal of reinforcement learning agent is to interact with a given environment and maximize its expected cumulative reward. The environment is usually modeled by a Markov Decision Process (MDP), given by tuples $\left<\mathcal{S}, \mathcal{A}, P, R, \gamma\right>,$ where $\mathcal{S}, \mathcal{A}$ represent the set of states and actions respectively. $P:\mathcal{S}\times\mathcal{A}\to \mathcal{S}$ is the transition function and $R:\mathcal{S}\times \mathcal{A}\to [0,1]$ is the reward function. $\gamma$ is the discount factor. The agent trys to find a policy $\pi:\mathcal{S}\to \mathcal{A}$ that maximizes its expected curriculum reward $V^{\pi}(s_0)$,  where $s_0=s$ is usually given or drawn from a distribution $\mu_0$  of initial state. The value function $V^{\pi}(s)$ is defined as
    \begin{align*}
        V^{\pi}(s)=\mathbb{E}_{s_0=s, a_t\sim \pi(\cdot\mid s_t), s_{t+1}\sim P(\cdot\mid s_t,a_t)}\left[\sum_{t=0}^{\infty}\gamma^tR(s_t,a_t)
        \right].
    \end{align*}

\textbf{Goal-oriented MDP}
    In this paper, we consider a specific class of MDP called goal-oriented MDP. 
    We use $\mathcal{G}$ to denote the set of goals.
    Different from traditional MDP, the reward function $R$ is a goal-conditioned sparse and binary signal indicating whether the goal is achieved: 
    \begin{align}\label{reward_setting}
        R_g(s_t,a_t,s_{t+1})&:=\left\{
            \begin{array}{cc}
                0, & \|\phi(s_{t+1})-g\|_2\leq\delta_g \\
                -1, & \text{otherwise}.
            \end{array}
        \right.
    \end{align}
    $\phi:\mathcal{S}\rightarrow\mathcal{G}$ is a known and tractable mapping that defines goal representation. $\delta_g$ is a given threshold indicating whether the goal is considered to be reached (see \cite{plappert2018multi}).
    
\textbf{Universal value function}
    The idea of universal value function is to use a single functional approximator, such as neural networks, to represent a large number of value functions. For the goal-oriented MDPs, the goal-based value function of a policy $\pi$ for any given goal $g$ is defined as $V^{\pi}(s,g)$, for all state $s\in \mathcal{S}.$ That is 
    \begin{align}
        V^\pi(s,g):=\mathbb{E}_{s_0=s, a_t\sim \pi(\cdot\mid s_t, g), s_{t+1}\sim P(\cdot\mid s_t,a_t)}\left[\sum_{t=0}^\infty\gamma^tR_g(s_t,a_t,s_{t+1})\right].\label{value_function0}
    \end{align}
    
    Let $\mathcal{T}^*:\mathcal{S}\times\mathcal{G}\to [0,1]$ be the joint distribution over starting state $s_0\in \mathcal{S}$ and goal $g\in \mathcal{G}.$. That is, at the start of every episode, a state-goal pair $(s_0,g)$ will be drawn from the task distribution $\mathcal{T}^*$. The agent tries to find a policy $\pi:\mathcal{S}\times \mathcal{G}\to \mathcal{A}$ that maximizes the expectation of discounted cumulative reward
    \begin{align}
        V^\pi(\mathcal{T}^*):=\mathop{\mathbb{E}}_{(s_0,g)\sim\mathcal{T}^*}[V^\pi(s_0,g)]\label{value_function}
    \end{align}

    Goal-oriented MDP characterizes several reinforcement benchmark tasks, such as the robotics tasks in the OpenAI gym environment \citep{plappert2018multi}. For example, in the FetchPush (see Figure \ref{show_her}) task, the agent needs to learn pushing a box to a designated point. In this task, the state of the system $s$ contains the status for both the robot and the box. The goal $g$, on the other hand, only indicates the designated position of the box. Thus, the mapping $\phi$ is defined as a mapping from a system state $s$ to the position of the box in $s$.

\textbf{Access to Simulator} 
    One of the common assumption made by previous work is an universal simulator that allows the environment to be reset to any given state \citep{florensa2017reverse,ecoffet2019go}. This kind of simulator is excessively powerful, and hard to build when acting in the real world. On the contrary, our method does not require an universal simulator, and thus is more realizable. 
\section{Related Work}
\textbf{Multi-Goal RL}
     The role of goal-conditioned policy has been investigated widely in deep reinforcement learning scenarios \citep{pong2019skew}. A few examples include grasping skills in imitation learning \citep{pathak2018zero,srinivas2018universal}, disentangling task knowledge from environment \citep{mao2018universal,ghosh2018learning}, and constituting lower-level controller in hierarchical RL \citep{oh2017zero,nachum2018data,huang2019mapping,eysenbach2019search}. By learning a universal value function which parameterizes the goal using a function approximator \citep{schaul2015universal}, an agent is able to learn multiple tasks simultaneously \citep{kaelbling1993learning,veeriah2018many} and identify important decision states \citep{goyal2019infobot}. It is shown that multi-task learning with goal-conditioned policy improves the generalizability to unseen goals (e.g., \cite{schaul2015universal}). 
     
\textbf{Hindsight Experience Replay}
    Hindsight Experience Replay \citep{andrychowicz2017hindsight} is an effective experience replay strategy which generates reward signal from failure trajectories. The idea of hindsight experience replay can be extended to various goal-conditioned problems, such as hierarchical RL \citep{levy2019learning}, dynamic goal pursuit \citep{fang2019dher}, goal-conditioned imitation \citep{ding2019goal,sun2019policy} and visual robotics applications \citep{nair2018visual,sahni2019addressing}. It is also shown that hindsight experience replay can be combined with on-policy reinforcement learning algorithms by importance sampling \citep{rauber2019hindsight}.
    
\textbf{Curriculum Learning in RL}
    Curriculum learning in RL usually suggests using a sequence of auxiliary tasks to guide policy optimization, which is also related to multi-task learning, lifelong learning, and transfer learning. The research interest in automatic curriculum design has seen rapid growth recently, where approaches have been proposed to schedule a given set of auxiliary tasks \citep{riedmiller2018learning,colas2019curious}, and to provide intrinsic motivation \citep{forestier2017intrinsically,pere2018unsupervised,sukhbaatar2018intrinsic,colas2018gep}. Generating goals which leads to high-value states could substantially improve the sample efficiency of RL agent \citep{goyal2018recall}. Guided exploration through curriculum generation is also an active research topic, where either the initial state \citep{florensa2017reverse} or the goal position \citep{baranes2013active,florensa2018automatic} is considered as a manipulable factor to generate the intermediate tasks. However, most curriculum learning methods are domain-specific, and it is still open to build  a generalized framework for curriculum learning.
\section{Automatic Hindsight Goal Generation}
\begin{wrapfigure}[12]{R}{4.5cm}
    \centering \includegraphics[height=3.0cm]{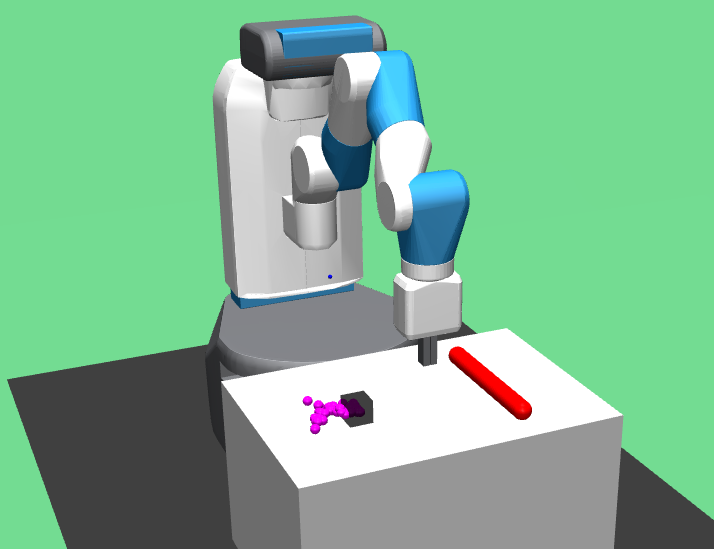}
   \caption{Visualization of hindsight goals (pink particles).}
    \label{show_her}
\end{wrapfigure}
As discussed in the previous section, HER provides an effective solution to resolve the sparse reward challenge in object manipulation tasks, in which achieved state in some past trajectories will be replayed as  imaginary goals. In the other words, HER modifies the task distribution in replay buffer to generate a set of auxiliary nearby goals which can used for further exploration and improve the performance of an off-policy RL agent which is expected to reach a very distant goal. However, the distribution of hindsight goals where the policy is trained on might differ significantly from the original task or goal distribution. Take Figure \ref{show_her} as an example, the desired goal distribution is lying on the red segment, which is far away from the initial position. In this situation, those hindsight goals may not be effective enough to promote policy optimization in original task. The goal of our work is to develop a new approach to generate valuable hindsight goals that will improve the performance on the original task.

In the rest of this section, we will present a new algorithmic framework as well as our implementation for automatic hindsight goal generation for better exploration.

\subsection{Algorithmic Framework}\label{sec:alg}

    Following \cite{florensa2018automatic}, our approach relies on the following generalizability assumption.     
    
    \begin{assumption}\label{generalizability_assumption}
        A value function of a policy $\pi$ for a specific goal $g$ has some generalizability to another goal $g'$ close to $g$.
    \end{assumption}
    
    One possible mathematical characterization for Assumption \ref{generalizability_assumption} is via the Lipschitz continuity. Similar assumptions have been widely applied in many scenarios \citep{asadi2018lipschitz,luo2019algorithmic}:
    \begin{equation}
        \left|V^{\pi}(s,g)-V^{\pi}(s',g')\right|\le L\cdot d((s,g),(s',g')), \label{discrepancy_bound}
    \end{equation}
    where $d((s,g),(s',g'))$ is a metric defined by
    \begin{align}\label{state_abstraction}
        d((s,g),(s',g'))=c\|\phi(s)-\phi(s')\|_2+\|g-g'\|_2.
    \end{align}
    for some hyperparameter $c>0$ that provides a trade-off between the distances between initial states and the distance between final goals. $\phi(\cdot)$ is a state abstraction to map from the state space to the goal space. When experimenting with the tasks in the OpenAI Gym environment \citep{plappert2018multi}, we simply adopt the state-goal mappings as defined in \eqref{reward_setting}.
    Although the Lipschitz continuity may not hold for every $s,s'\in \mathcal{S}, g,g'\in \mathcal{G}, $ we only require continuity over some specific region. It is reasonable to claim that bound Eq.~\eqref{discrepancy_bound} holds for most of the $(s,g),(s',g')$ when $d((s,g),(s',g'))$ is not too large.

    Partly due to the reward sparsity of the distant goals, optimizing the expected cumulative reward (see Eq.~\eqref{value_function}) from scratch is very difficult. Instead, we propose to optimize a relaxed lower bound which introduces intermediate goals that may be  easier to optimize. Here we provide Theorem \ref{lem:lower_bound} that establishes the such a lower bound.
    
    \begin{restatable}{theorem}{thmbound}\label{lem:lower_bound}
        Assuming that the generalizability condition (Eq.~\eqref{discrepancy_bound}) 
        holds for two distributions $(s,g)\sim \mathcal{T}$ and $(s',g')\sim \mathcal{T}'$, we have
        \begin{equation}\label{equ:wasserstein_lowerbound}
            V^\pi(\mathcal{T}')\ge V^\pi(\mathcal{T})-L\cdot D(\mathcal{T},\mathcal{T}'). 
        \end{equation}
        where $D(\cdot,\cdot)$ is the Wasserstein distance based on $d(\cdot,\cdot)$
        \begin{align*}
            D(\mathcal{T}^{(1)},\mathcal{T}^{(2)})=&\inf_{\mu\in\Gamma(\mathcal{T}^{(1)},\mathcal{T}^{(2)})} \left({\mathbb{E}}_{\mu}[d(({s_0}^{(1)},g^{(1)}),({s_0}^{(2)},g^{(2)}))]\right)
        \end{align*}
        where $\Gamma(\mathcal{T}^{(1)},\mathcal{T}^{(2)})$ denotes the collection of all joint distribution $\mu({s_0}^{(1)},g^{(1)},{s_0}^{(2)},g^{(2)})$ whose marginal probabilities are  $\mathcal{T}^{(1)},\mathcal{T}^{(2)}$, respectively.
    \end{restatable}
    The proof of Theorem 1 is deferred to Appendix \ref{app:proof}.
    
    It follows from Theorem 1 that optimizing cumulative rewards Eq.~\eqref{value_function} can be relaxed into the following surrogate problem
    \begin{align}\label{optimization}
        \max_{\mathcal{T}, \pi} \quad V^\pi(\mathcal{T})-L\cdot D(\mathcal{T},\mathcal{T}^*).
    \end{align}
    
    Note that this new objective function is very intuitive. Instead of optimizing with the difficult goal/task distribution $\mathcal{T}^*$, we hope to find a collection of surrogate goals $\mathcal{T}$, which are both easy to optimize and are also close or converging towards $\mathcal{T}^*$. 
    However the joint optimization of $\pi$ and $\mathcal{T}$ is non-trivial. This is because a) $\mathcal{T}$ is a high-dimensional distribution over tasks, b) policy $\pi$ is optimized with respect to a shifting task distribution $\mathcal{T}$, c) the estimation of value function $V^\pi(\mathcal{T})$ may not be quite accurate during training.
    
    Inspired by \citet{andrychowicz2017hindsight}, we adopt the idea of using hindsight goals here. We first enforce $\mathcal{T}$ to be a finite set of $K$ particles which can only be from those already achieved states/goals from the replay buffer $B$. 
    In another word, the support of the set $\mathcal{T}$ should lie inside $B$.
    In the meanwhile, we notice that a direct implementation of problem Eq.~\eqref{optimization} may lead to degeneration of hindsight goal selection of the training process, i.e., the goals may be all drawn from a single trajectory, thus not being able to provide sufficient exploration. 
    Therefore, we introduce an extra diversity constraint, i.e, for every trajectory $\tau\in B$, at most $\mu$ states can be selected in $\mathcal{T}$. In practice, we find that simply setting it to 1 would result in reasonable performance. It is shown in Section \ref{sec:abliation} that this diversity constraint indeed improves the robustness of our algorithm.
    
    Finally, the optimization problem we aim to solve is,
    \begin{align*}
     \max_{\pi, \mathcal{T}:|\mathcal{T}|=K} \quad &{V}^\pi(\mathcal{T})-L\cdot D(\mathcal{T},\mathcal{T}^*)\\
        \text{s.t. } \quad 
        &\sum_{s_0,s_t\in \tau} \mathbbm{1}[(s_0,\phi(s_t)) \in \mathcal{T}]\le 1, \quad\forall \tau\in B \\
        &\sum_{\tau \in B}\sum_{s_0,s_t\in \tau} \mathbbm{1}[(s_0,\phi(s_t)) \in \mathcal{T}] = K.
    \end{align*}
    
    To solve the above optimization, we adapt a two-stage iterative algorithm. First, we apply a policy optimization algorithm, for example DDPG, to maximize the value function conditioned on the task set $\mathcal{T}$. Then we fix $\pi$ and optimize the the hindsight set $\mathcal{T}$ subject to the diversity constraint, which is a variant of the well-known Wasserstein Barycenter problem with a bias term (the value function) for each particle. Then we iterate the above process until the policy achieves a desirable performance or we reach a computation budget. It is not hard to see that the first optimization of value function  is straightforward. In our work, we simply use the DDPG+HER framework for it. The second optimization of hindsight goals is non-trivial. In the following, we describe an efficient approximation algorithm for it.

\subsection{Solving  Wasserstein Barycenter Problem via Bipartite Matching}

    Since we assume that $\mathcal{T}$ is hindsight and with $K$ particles, we can approximately solve the above Wasserstein Barycenter problem in the combinatorial setting as a bipartite matching problem. Instead of dealing with $\mathcal{T}^*$,  we draw $K$ samples from $\mathcal{T}^*$ to empirically approximate it by a set of $K$ particles $\widehat{\mathcal{T}}^*$. 
    In this way, the hindsight task set $\mathcal{T}$ can be solved in the following way. For every task instance $(\hat s^i_0,\hat g^i)\in\widehat{\mathcal{T}}^*$, we find a state trajectory $\tau^i=\{s^i_t\}\in B$ that together minimizes the sum
    \begin{align}
     \sum_{(\hat s^i_0,\hat g^i)\in\widehat{\mathcal{T}}^*}w((\hat s^i_0,\hat g^i),\tau^i) \label{equ:distance_to_trajectory}
    \end{align}
    where we define
    \begin{align}\label{equ:def-w}
        w((\hat s^i_0,\hat g^i),\tau^i) := c\|\phi(\hat s^i_0)-\phi(s^i_0)\|_2 + \min_{t}\left(\|\hat g^i-\phi(s^i_t)\|_2-\frac{1}{L}V^\pi(s^i_0,\phi(s^i_t))\right).
    \end{align}
    Finally we select each corresponding achieved state $s_t\in\tau$ to construct hindsight goal $(\hat s_0,\phi(s_t))\in \mathcal{T}$. It is not hard to see that the above combinatorial optimization  exactly identifies optimal solution $\mathcal{T}$ in the above-mentioned  Wasserstein Barycenter problem. In practice, the Lipschitz constant $L$ is unknown and therefore treated as a hyper-parameter. 
    
    
    The optimal solution of the combinatorial problem in Eq.~\eqref{equ:distance_to_trajectory} can be solved efficiently by the well-known maximum weight bipartite matching \citep{munkres1957algorithms,duan2012scaling}. The bipartite graph $G(\{V_x, V_y\}, E)$ is constructed as follows. Vertices are split into two partitions $V_x, V_y$. Every vertex in $V_x$ represents a task instance $(\hat{s}_0, \hat{g})\in \hat{\mathcal{T}}^*$, and vertex in $V_y$ represents a  trajectory $\tau\in B$. The weight of edge connecting $(\hat{s}_0, \hat{g})$ and $\tau$ is $-w((\hat{s}_0, \hat{g}), \tau)$ as defined in Eq.~\eqref{equ:def-w}. 
    In this paper, we apply the Minimum Cost Maximum Flow algorithm to solve this bipartite matching problem (for example, see \cite{networkflow}). 
    

    \begin{algorithm}[htp]
    \caption{Exploration via Hindsight Goal Generation (HGG)}
    \label{main_alg}
        \begin{algorithmic}[1]
            \State Initialize $\pi$ \Comment{initialize neural networks}
            \State $B\gets\emptyset$
            \For{iteration $=1,2,\dots, N$}
                \State Sample $\{(\hat s_0^i,\hat g^i)\}_{i=1}^K\sim\mathcal{T}^*$ \Comment{sample from target distribution}
                \State Find $K$ distinct trajectories $\{\tau^i\}_{i=1}^K$ that minimize \Comment{weighted bipartite matching}
                \begin{align*}
                    \sum_{i=1}^K w((\hat s_0^i,\hat g^i),\tau^i) &= \sum_{i=1}^K \left(c\|\phi(\hat s_0^i)-\phi(s_0^i)\|_2 + \min_{t}\left(\|\hat g^i-\phi(s_t^i)\|_2-\frac{1}{L}V^{\pi}(s^i_0,\phi(s_t^i))\right)\right)
                \end{align*}
                \State Construct intermediate task distribution $\{(\hat s_0^i,g^i)\}_{i=1}^M$ where
                \begin{align*}
                    g^i &= \phi\left(\mathop{\arg\min}_{s_t^i\in  \tau_i}\left(\|\hat g^i-\phi(s_t^i)\|_2-\frac{1}{L}V^{\pi}(s^i_0,\phi(s_t^i))\right)\right)
                \end{align*}
                \For{$i=1,2,\dots, K$}
                    \State $(s_0,g)\gets(\hat s_0^i,g^i)$ \Comment{\textbf{critical step: hindsight goal-oriented exploration}} \label{alg_goal_generating}
                    \For{$t=0,1,\dots, H-1$}
                        \State $a_t\gets\pi(\cdot|s_t,g)+\text{noise}$ \Comment{together with $\epsilon$-greedy or Gaussian exploration}
                        \State $s_{t+1}\sim P(\cdot|s_t,a_t)$
                        \State $r_t\gets R_g(s_t,a_t,s_{t+1})$
                    \EndFor
                    \State $\tau\gets\{s_0,a_0,r_0,s_1,\dots\}$
                    \State $B\gets B\cup\{\tau\}$
                \EndFor
                \For{$i=1\dots M$}
                    \State Sample a minibatch $b$ from replay buffer using HER
                    \label{alg_replay_strategy}
                    \State Perform one step on value and policy update on minibatch $b$ using DDPG
                \EndFor
            \EndFor
        \end{algorithmic}
    \end{algorithm}
    
    \textbf{Overall Algorithm} The overall description of our algorithm is shown in Algorithm \ref{main_alg}.
    Note that our  exploration strategy the only modification is in Step \ref{alg_goal_generating}, in which we generate hindsight goals to guide the agent to collect more valuable trajectories. So it is complementary to other improvements in DDPG/HER around Step $\ref{alg_replay_strategy}$, such as the prioritized experience replay strategy \citep{schaul2016prioritized,zhao2018energy,zhao2019maximum} and other variants of hindsight experience replay \citep{fang2019curriculum,bai2019guided}.
\section{Experiments}

    
    Our experiment environments are based on the standard robotic manipulation environments in the OpenAI Gym \citep{openaigym}\footnote{Our code is available at \url{https://github.com/Stilwell-Git/Hindsight-Goal-Generation}.}. In addition to the standard settings, to better visualize the improvement of the sample efficiency, we vary the target task distributions in the following ways:
    \begin{itemize}
        \item Fetch environments: Initial object position and goal are generated uniformly at random from two distant segments.
        \item Hand-manipulation environments : These tasks require the agent to rotate the object into a given pose, and only the rotations around $z$-axis are considered here. We restrict the initial axis-angle in a small interval, and the target pose will be generated in its symmetry. That is, the object needs to be rotated in about $\pi$ degree. 
        \item Reach environment: FetchReach and HandReach do not support randomization of the initial state, so we restrict their target distribution to be a subset of the original goal space.
    \end{itemize}
    
    Regarding baseline comparison, we consider the original DDPG+HER algorithm. We also investigate the integration of the  experience replay prioritization strategies, such as the Energy-Based Prioritization (EBP) proposed by \citet{zhao2018energy}, which draws the prior knowledge of physics system to exploit valuable trajectories.
    More details of experiment settings are included in the Appendix \ref{experiment_settings}. 

\subsection{HGG Generates Better Hindsight Goals for Exploration}
    \begin{figure}[htp]
        \centering
        \begin{subfigure}[b]{0.1\textwidth}
            HGG
            \vspace{0.55in}
            \\
            HER
            \vspace{0.5in}
        \end{subfigure}
        \begin{subfigure}[b]{0.19\textwidth}
            \includegraphics[width=\textwidth]{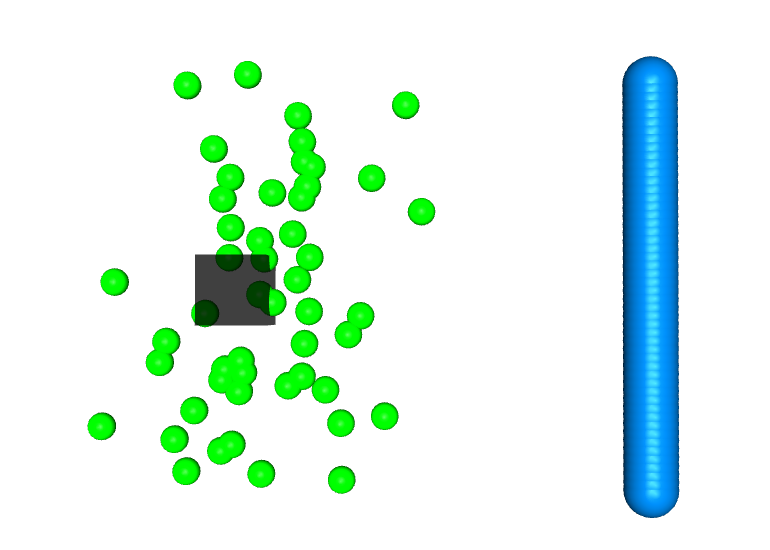}
            \includegraphics[width=\textwidth]{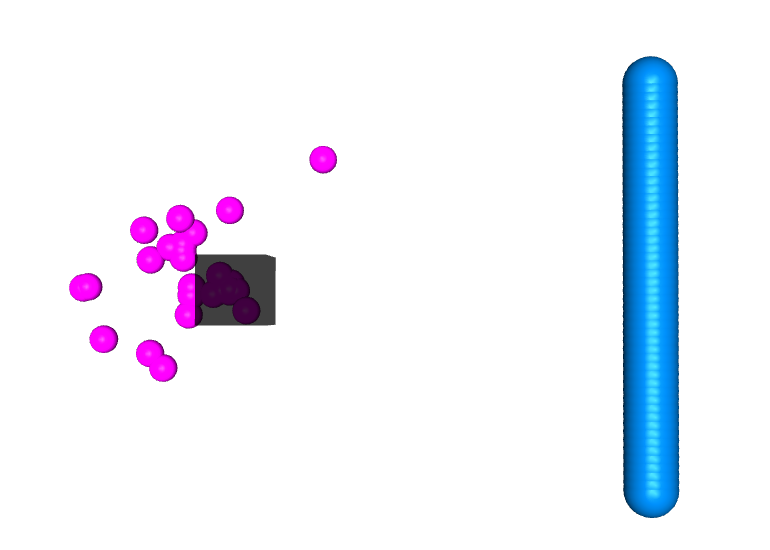}
            \caption{Episode 500}
        \end{subfigure}
        \begin{subfigure}[b]{0.19\textwidth}
            \includegraphics[width=\textwidth]{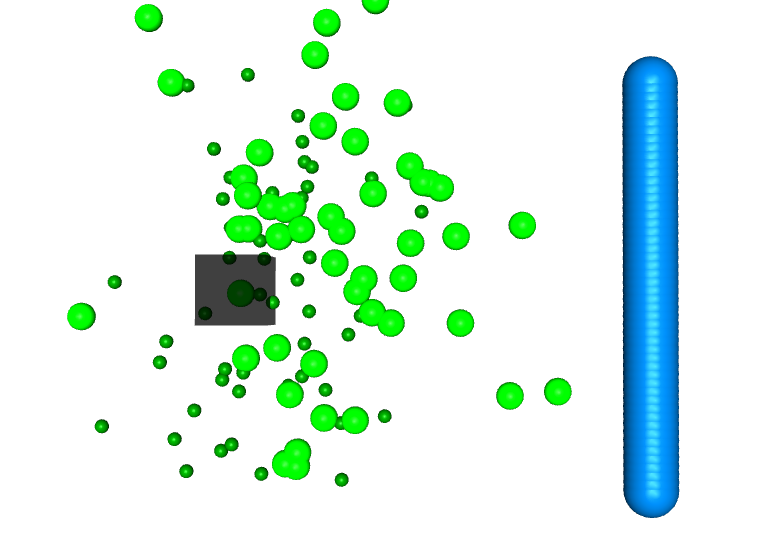}
            \includegraphics[width=\textwidth]{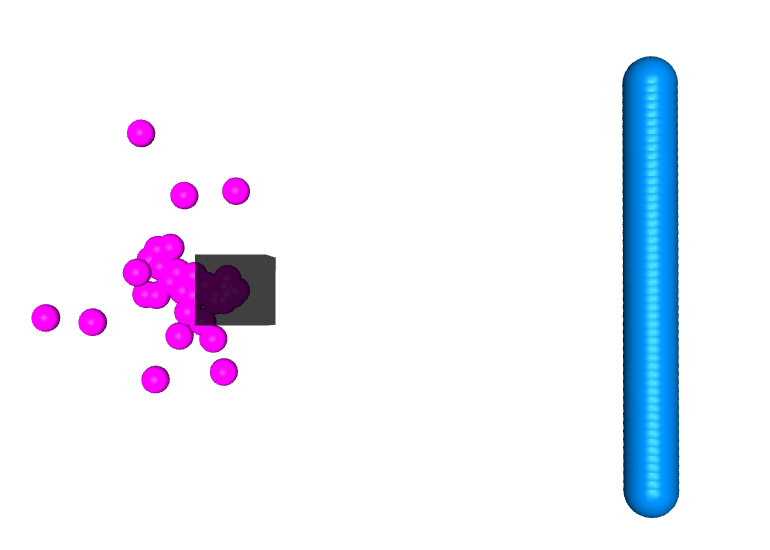}
            \caption{Episode 1000}
        \end{subfigure}
        \begin{subfigure}[b]{0.19\textwidth}
            \includegraphics[width=\textwidth]{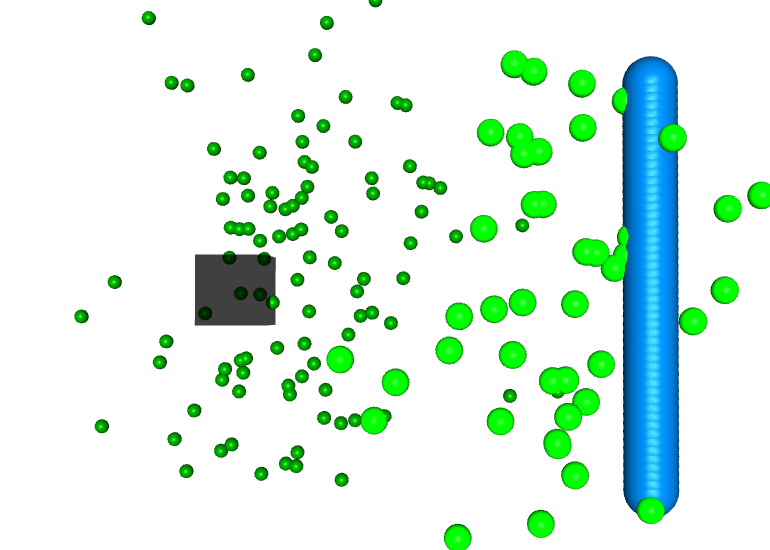}
            \includegraphics[width=\textwidth]{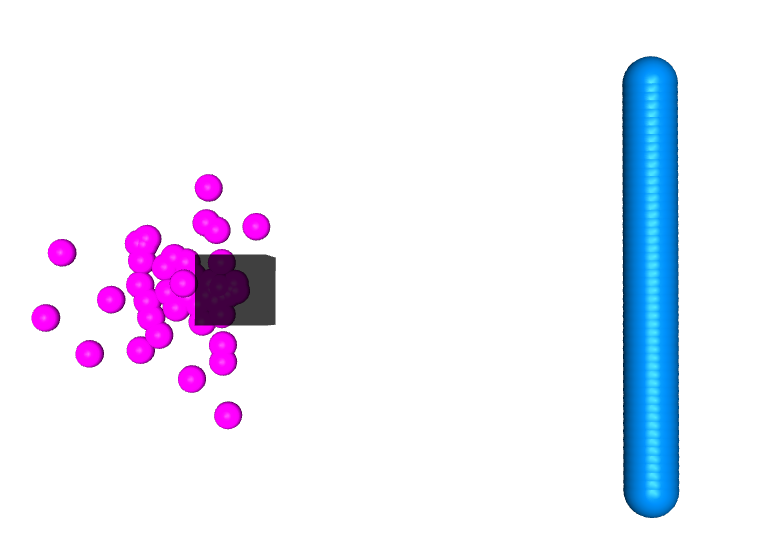}
            \caption{Episode 2000}
        \end{subfigure}
        \begin{subfigure}[b]{0.19\textwidth}
            \includegraphics[width=\textwidth]{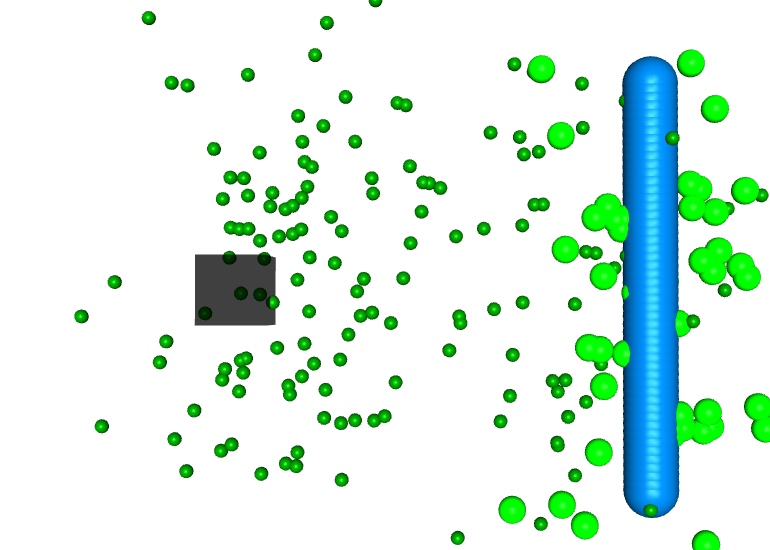}
            \includegraphics[width=\textwidth]{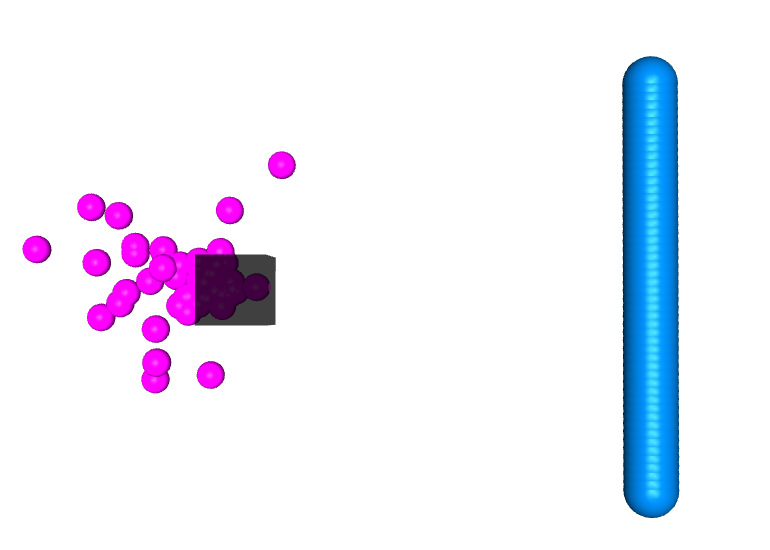}
            \caption{Episode 3000}
        \end{subfigure}
        \caption{Visualization of goal distribution generated by HGG and HER on FetchPush. The initial object position is shown as a black box. The blue segment indicates target goal distribution. The above row presents the distribution of the hindsight goals generated by our HGG  method, where bright green particles is a batch of recently generated goals, and dark green particles present the goals generated in the previous iterations. The bottom row presents the distribution of replay goals generated by HER.}
        \label{visualization}
    \end{figure}
    We first check whether HGG is able to generate meaningful hindsight goals for exploration. We compare HGG and HER in the FetchPush environment. 
    It is shown in Figure \ref{visualization} that HGG algorithm generates goals that gradually move towards the target region. Since those goals are hindsight, they are considered to be achieved during training. In comparison, the replay distribution of a DDPG+HER agent has been stuck around the initial position for many iterations, indicating that those goals may not be able to efficiently guide exploration.

\textbf{Performance on benchmark robotics tasks}
    \begin{figure}[h]
        \centering
        \includegraphics[width=0.95\textwidth]{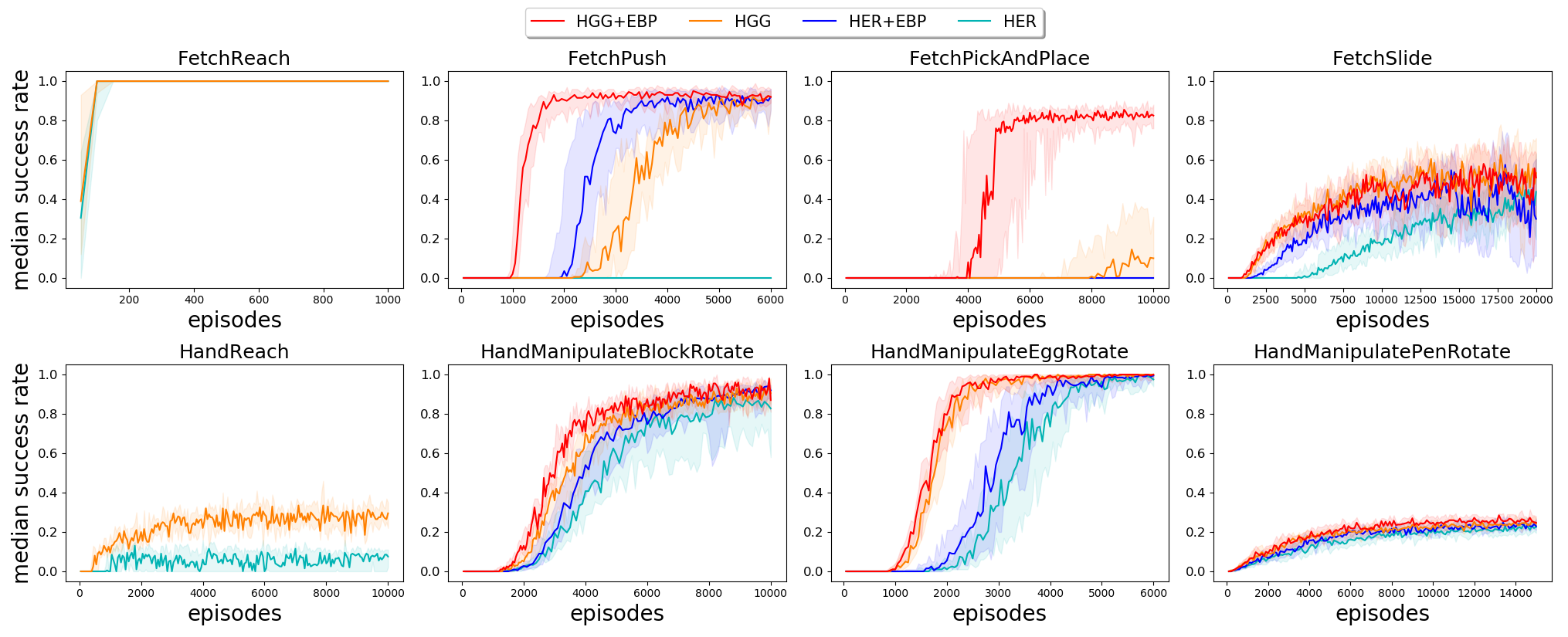}
        \caption{Learning curves for variant a number of goal-oriented robotic manipulation tasks. All curves presented in this figure are trained with default hyper-parameters included in Appendix \ref{hyper_parameters}. Note that since FetchReach and HandReach do not contain object instances for EBP, so we do not include the +EBP versions for them.}
        \label{main_results}
    \end{figure}
    
Then we check whether the exploration provided by the goals generated by HGG can result in better policy training performance. As shown in Figure \ref{main_results}, we compare the vanilla HER, HER with Energy-Based Prioritization (HER+EBP), HGG, HGG+EBP. It is worth noting that since EBP is designed for the Bellman equation updates, it is complementary to our HGG-based exploration approach. Among the eight environments, HGG substantially outperforms HER on four and has comparable performance on the other four, which are either too simple or too difficult. When combined with EBP, HGG+EBP achieves the best performance on six environments that are eligible. 

\begin{wrapfigure}[10]{R}{7.0cm}
    \centering
    \vspace{-0.18in}
    \begin{minipage}{0.23\textwidth}
        \centering\vspace{-0.1in}
        \includegraphics[height=2.4cm]{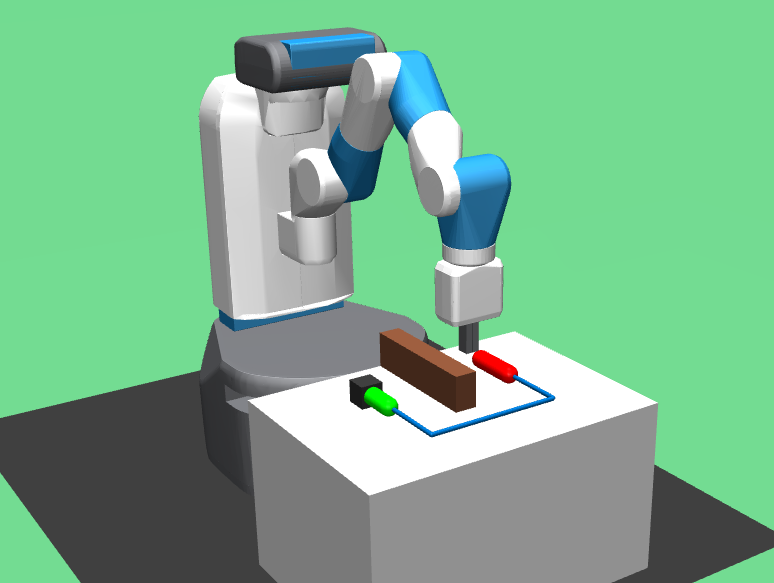}
    \end{minipage}
    \begin{minipage}{0.25\textwidth}
        \centering
        \includegraphics[height=2.8cm]{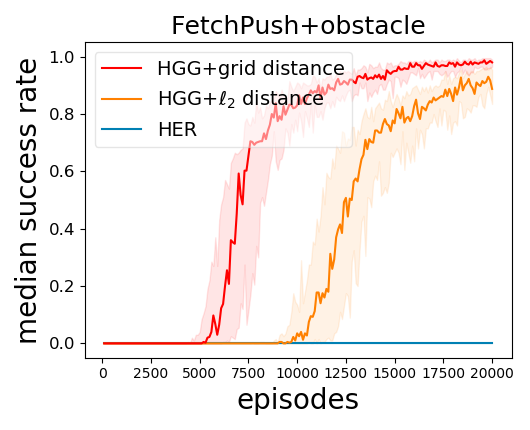}
    \end{minipage}
    \caption{Visualization of FetchPush with obstacle.}
    \label{obstacle_task}
\end{wrapfigure}
\textbf{Performance on tasks with obstacle} In a more difficult task, crafted metric may be more suitable than $\ell_2$-distance used in Eq.~\eqref{state_abstraction}. As shown in Figure \ref{obstacle_task}, we created an environment based on FetchPush with a rigid obstacle. The object and the goal are uniformly generated in the green and the red segments respectively. The brown block is a static wall which cannot be moved. In addition to $\ell_2$, we also construct a distance metric based on the graph distance of a mesh grid on the plane, the blue line is a successful trajectory in such hand-craft distance measure. A more detailed description is deferred to Appendix \ref{grid_metric_details}. Intuitively speaking, this crafted distance should be better than $\ell_2$ due to the existence of the obstacle. Experimental results suggest that such a crafted distance metric provides better guidance for goal generation and training, and significantly improves sample efficiency over $\ell_2$ distance. It would be a future direction to investigate ways to obtain or learn a good metric.

\subsection{Comparison with Explicit Curriculum Learning}\label{compare_curriculum}
    \begin{wrapfigure}[13]{R}{5cm}
        \centering
        \vspace*{-0.3in}
        \includegraphics[height=3.5cm]{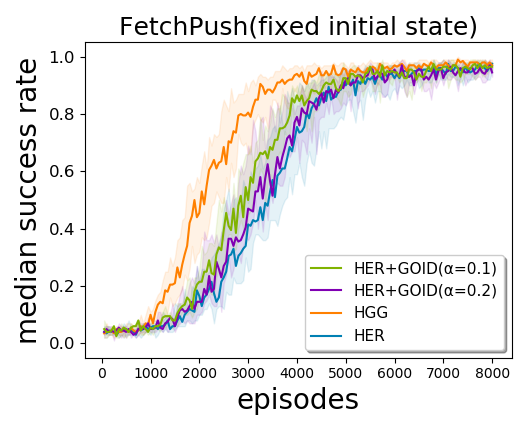}
        \caption{Comparison with curriculum learning. We compare HGG with the original HER, HER+GOID with two threshold values.}
        \label{rejection_sampling}
    \end{wrapfigure}
    
    Since our method can be seen as an explicit curriculum learning for exploration, where we generate hindsight goals as intermediate task distribution, we also compare our method with another recently proposed curriculum learning method for RL. \citet{florensa2018automatic} leverages Least-Squares GAN \citep{mao2018effectiveness} to mimic the set called Goals of Intermediate Difficult as exploration goal generator.
    
    Specifically, in our task settings, we define a goal set
    $
        GOID(\pi) = \{g:\alpha\leq f(\pi,g) \leq 1-\alpha\},
    $
    where $f(\pi,g)$ represents the average success rate in a small region closed by goal $g$. 
    To sample from $GOID$, we implement an oracle goal generator based on rejection sampling, which could uniformly sample goals from $GOID(\pi)$. Result in Figure \ref{rejection_sampling} indicates that our Hindsight Goal Generation substantially outperforms HER even with $GOID$ from the oracle generator. Note that this experiment is run on a environment with fixed initial state due to the limitation of \citet{florensa2018automatic}. The choice of $\alpha$ is also suggested by \citet{florensa2018automatic}.
    

\subsection{Ablation Studies on Hyperparameter Selection}
\label{sec:abliation}

    In this section, we set up a set of ablation tests on several hyper-parameters used in the Hindsight Goal Generation algorithm.
    
    \textbf{Lipschitz $L$:} The selection of Lipschitz constant is task dependent, since it iss related with scale of value function and goal distance. For the robotics tasks tested in this paper, we find that it is easier to set $L$ by first divided it with the upper bound of the distance between any two final goals in a environment. We test a few choices of $L$ on several environments and find that it is very easy to find a range of $L$ that works well and shows robustness for all the environments tested in this section. We show the learning curves on FetchPush with different $L$. It appears that the performance of HGG is reasonable as long as $L$ is not too small. For all tasks we tested in the comparisons, we set $L=5.0$.
    
    \textbf{Distance weight $c$:} Parameter $c$ defines the trade-off between the initial state similarity and the goal similarity.  Larger $c$ encourages our algorithm to choose hindsight goals that has closer initial state. Results in Figure \ref{ablation_study} indicates that the choice of $c$ is indeed robust. For all tasks we tested in the comparisons, we set $c=3.0$.
    
    
    \textbf{Number of hindsight goals $K$:} We find that for the simple tasks, the choice of $K$ is not critical. Even a greedy approach (corresponds to $K=1$) can achieved competitive performance, e.g. on FetchPush in the third panel of Figure \ref{ablation_study}. For more difficult environment, such as FetchPickAndPlace, larger batch size can significantly reduce the variance of training results. For all tasks tested in the comparisons, we ploted the best results given by $K\in\{50,100\}$.
    \begin{figure}[htp]
        \centering
        \includegraphics[width=0.95\textwidth]{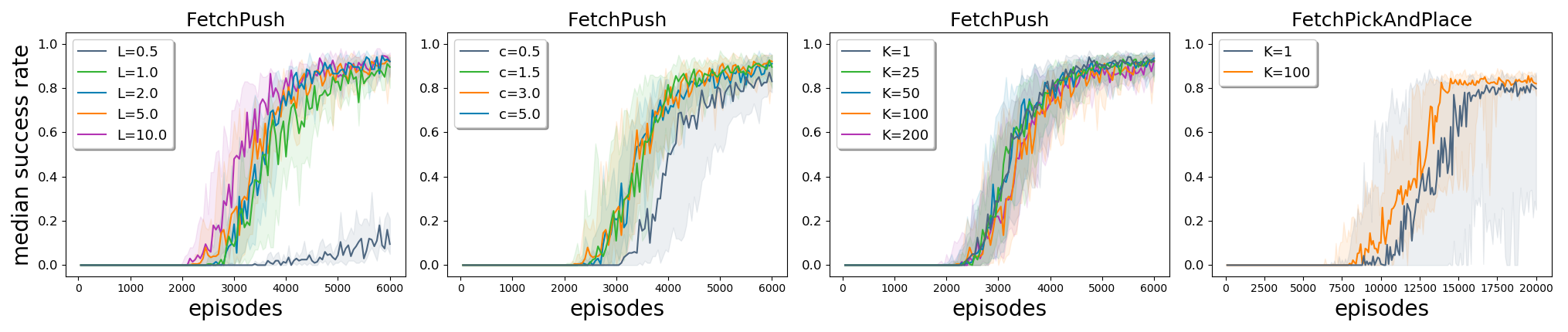}
        \caption{Ablation study of hyper-parameter selection. Several curves are omitted in the forth panel to provide a clear view of variance comparison. A full version is deferred to Appendix \ref{appendix_ablation}.}
        \label{ablation_study}
    \end{figure}
\section{Conclusion}

We present a novel automatic hindsight goal generation algorithm, by which valuable hindsight imaginary tasks are generated to enable efficient exploration for goal-oriented off-policy reinforcement learning. We formulate this idea as a surrogate optimization to identify hindsight goals that are easy to achieve and also likely to lead to the actual goal. We introduce a combinatorial solver to generate such intermediate tasks. Extensive experiments demonstrated better goal-oriented exploration of our method over original HER and curriculum learning on a collection of robotic learning tasks. A future direction is to incorporate the controllable representation learning \citep{thomas2017independently} to provide task-specific distance metric \citep{ghosh2018learning, srinivas2018universal}, which may generalize our method to more complicated cases where the standard Wasserstein distance cannot be applied directly.

\bibliographystyle{plainnat}
\bibliography{ref}

\begin{thebibliography}{55}
\providecommand{\natexlab}[1]{#1}
\providecommand{\url}[1]{\texttt{#1}}
\expandafter\ifx\csname urlstyle\endcsname\relax
  \providecommand{\doi}[1]{doi: #1}\else
  \providecommand{\doi}{doi: \begingroup \urlstyle{rm}\Url}\fi

\bibitem[Ahuja et~al.(1993)Ahuja, Magnanti, and Orlin]{networkflow}
Ravindra~K. Ahuja, Thomas~L. Magnanti, and James~B. Orlin.
\newblock \emph{Network Flows: Theory, Algorithms, and Applications}.
\newblock Prentice-Hall, Inc., Upper Saddle River, NJ, USA, 1993.
\newblock ISBN 0-13-617549-X.

\bibitem[Andrychowicz et~al.(2017)Andrychowicz, Wolski, Ray, Schneider, Fong,
  Welinder, McGrew, Tobin, Abbeel, and Zaremba]{andrychowicz2017hindsight}
Marcin Andrychowicz, Filip Wolski, Alex Ray, Jonas Schneider, Rachel Fong,
  Peter Welinder, Bob McGrew, Josh Tobin, OpenAI~Pieter Abbeel, and Wojciech
  Zaremba.
\newblock Hindsight experience replay.
\newblock In \emph{Advances in Neural Information Processing Systems}, pages
  5048--5058, 2017.

\bibitem[Asadi et~al.(2018)Asadi, Misra, and Littman]{asadi2018lipschitz}
Kavosh Asadi, Dipendra Misra, and Michael Littman.
\newblock Lipschitz continuity in model-based reinforcement learning.
\newblock In \emph{International Conference on Machine Learning}, pages
  264--273, 2018.

\bibitem[Bai et~al.(2019)Bai, Liu, Zhao, and Tang]{bai2019guided}
Chenjia Bai, Peng Liu, Wei Zhao, and Xianglong Tang.
\newblock Guided goal generation for hindsight multi-goal reinforcement
  learning.
\newblock \emph{Neurocomputing}, 2019.

\bibitem[Baranes and Oudeyer(2013)]{baranes2013active}
Adrien Baranes and Pierre-Yves Oudeyer.
\newblock Active learning of inverse models with intrinsically motivated goal
  exploration in robots.
\newblock \emph{Robotics and Autonomous Systems}, 61\penalty0 (1):\penalty0
  49--73, 2013.

\bibitem[Brockman et~al.(2016)Brockman, Cheung, Pettersson, Schneider,
  Schulman, Tang, and Zaremba]{openaigym}
Greg Brockman, Vicki Cheung, Ludwig Pettersson, Jonas Schneider, John Schulman,
  Jie Tang, and Wojciech Zaremba.
\newblock Openai gym, 2016.

\bibitem[Colas et~al.(2018)Colas, Sigaud, and Oudeyer]{colas2018gep}
C{\'e}dric Colas, Olivier Sigaud, and Pierre-Yves Oudeyer.
\newblock Gep-pg: Decoupling exploration and exploitation in deep reinforcement
  learning algorithms.
\newblock In \emph{International Conference on Machine Learning}, pages
  1038--1047, 2018.

\bibitem[Colas et~al.(2019)Colas, Oudeyer, Sigaud, Fournier, and
  Chetouani]{colas2019curious}
C{\'e}dric Colas, Pierre-Yves Oudeyer, Olivier Sigaud, Pierre Fournier, and
  Mohamed Chetouani.
\newblock Curious: Intrinsically motivated modular multi-goal reinforcement
  learning.
\newblock In \emph{International Conference on Machine Learning}, pages
  1331--1340, 2019.

\bibitem[Ding et~al.(2019)Ding, Florensa, Abbeel, and Phielipp]{ding2019goal}
Yiming Ding, Carlos Florensa, Pieter Abbeel, and Mariano Phielipp.
\newblock Goal-conditioned imitation learning.
\newblock In \emph{Advances in Neural Information Processing Systems}, 2019.

\bibitem[Duan and Su(2012)]{duan2012scaling}
Ran Duan and Hsin-Hao Su.
\newblock A scaling algorithm for maximum weight matching in bipartite graphs.
\newblock In \emph{Proceedings of the twenty-third annual ACM-SIAM symposium on
  Discrete Algorithms}, pages 1413--1424. Society for Industrial and Applied
  Mathematics, 2012.

\bibitem[Ecoffet et~al.(2019)Ecoffet, Huizinga, Lehman, Stanley, and
  Clune]{ecoffet2019go}
Adrien Ecoffet, Joost Huizinga, Joel Lehman, Kenneth~O Stanley, and Jeff Clune.
\newblock Go-explore: a new approach for hard-exploration problems.
\newblock \emph{arXiv preprint arXiv:1901.10995}, 2019.

\bibitem[Eysenbach et~al.(2019)Eysenbach, Salakhutdinov, and
  Levine]{eysenbach2019search}
Benjamin Eysenbach, Ruslan Salakhutdinov, and Sergey Levine.
\newblock Search on the replay buffer: Bridging planning and reinforcement
  learning.
\newblock In \emph{Advances in Neural Information Processing Systems}, 2019.

\bibitem[Fang et~al.(2019{\natexlab{a}})Fang, Zhou, Shi, Gong, Xu, and
  Zhang]{fang2019dher}
Meng Fang, Cheng Zhou, Bei Shi, Boqing Gong, Jia Xu, and Tong Zhang.
\newblock Dher: Hindsight experience replay for dynamic goals.
\newblock In \emph{International Conference on Learning Representations},
  2019{\natexlab{a}}.

\bibitem[Fang et~al.(2019{\natexlab{b}})Fang, Zhou, Du, Han, and
  Zhang]{fang2019curriculum}
Meng Fang, Tianyi Zhou, Yali Du, Lei Han, and Zhengyou Zhang.
\newblock Curriculum-guided hindsight experience replay.
\newblock In \emph{Advances in Neural Information Processing Systems},
  2019{\natexlab{b}}.

\bibitem[Florensa et~al.(2017)Florensa, Held, Wulfmeier, Zhang, and
  Abbeel]{florensa2017reverse}
Carlos Florensa, David Held, Markus Wulfmeier, Michael Zhang, and Pieter
  Abbeel.
\newblock Reverse curriculum generation for reinforcement learning.
\newblock In \emph{Conference on Robot Learning}, pages 482--495, 2017.

\bibitem[Florensa et~al.(2018)Florensa, Held, Geng, and
  Abbeel]{florensa2018automatic}
Carlos Florensa, David Held, Xinyang Geng, and Pieter Abbeel.
\newblock Automatic goal generation for reinforcement learning agents.
\newblock In \emph{International Conference on Machine Learning}, pages
  1514--1523, 2018.

\bibitem[Forestier et~al.(2017)Forestier, Mollard, and
  Oudeyer]{forestier2017intrinsically}
S{\'e}bastien Forestier, Yoan Mollard, and Pierre-Yves Oudeyer.
\newblock Intrinsically motivated goal exploration processes with automatic
  curriculum learning.
\newblock \emph{arXiv preprint arXiv:1708.02190}, 2017.

\bibitem[Ghosh et~al.(2019)Ghosh, Gupta, and Levine]{ghosh2018learning}
Dibya Ghosh, Abhishek Gupta, and Sergey Levine.
\newblock Learning actionable representations with goal-conditioned policies.
\newblock In \emph{International Conference on Learning Representations}, 2019.

\bibitem[Goyal et~al.(2019{\natexlab{a}})Goyal, Brakel, Fedus, Singhal,
  Lillicrap, Levine, Larochelle, and Bengio]{goyal2018recall}
Anirudh Goyal, Philemon Brakel, William Fedus, Soumye Singhal, Timothy
  Lillicrap, Sergey Levine, Hugo Larochelle, and Yoshua Bengio.
\newblock Recall traces: Backtracking models for efficient reinforcement
  learning.
\newblock In \emph{International Conference on Learning Representations},
  2019{\natexlab{a}}.

\bibitem[Goyal et~al.(2019{\natexlab{b}})Goyal, Islam, Strouse, Ahmed,
  Botvinick, Larochelle, Levine, and Bengio]{goyal2019infobot}
Anirudh Goyal, Riashat Islam, Daniel Strouse, Zafarali Ahmed, Matthew
  Botvinick, Hugo Larochelle, Sergey Levine, and Yoshua Bengio.
\newblock Infobot: Transfer and exploration via the information bottleneck.
\newblock In \emph{International Conference on Learning Representations},
  2019{\natexlab{b}}.

\bibitem[Huang et~al.(2019)Huang, Su, and Liu]{huang2019mapping}
Zhiao Huang, Hao Su, and Fangchen Liu.
\newblock Mapping state space using landmarks for universal goal reaching.
\newblock In \emph{Advances in Neural Information Processing Systems}, 2019.

\bibitem[Kaelbling(1993)]{kaelbling1993learning}
Leslie~Pack Kaelbling.
\newblock Learning to achieve goals.
\newblock In \emph{IJCAI}, pages 1094--1099. Citeseer, 1993.

\bibitem[Karatzoglou et~al.(2013)Karatzoglou, Baltrunas, and
  Shi]{Karatzoglou2013}
Alexandros Karatzoglou, Linas Baltrunas, and Yue Shi.
\newblock Learning to rank for recommender systems.
\newblock In \emph{Proceedings of the 7th ACM conference on Recommender
  systems}, pages 493--494. ACM, 2013.

\bibitem[Levine et~al.(2016)Levine, Finn, Darrell, and Abbeel]{Levine2016}
Sergey Levine, Chelsea Finn, Trevor Darrell, and Pieter Abbeel.
\newblock End-to-end training of deep visuomotor policies.
\newblock \emph{The Journal of Machine Learning Research}, 17\penalty0
  (1):\penalty0 1334--1373, 2016.

\bibitem[Levy et~al.(2019)Levy, Konidaris, Platt, and Saenko]{levy2019learning}
Andrew Levy, George Konidaris, Robert Platt, and Kate Saenko.
\newblock Learning multi-level hierarchies with hindsight.
\newblock In \emph{International Conference on Learning Representations}, 2019.

\bibitem[Lillicrap et~al.(2016)Lillicrap, Hunt, Pritzel, Heess, Erez, Tassa,
  Silver, and Wierstra]{lillicrap2016continuous}
Timothy~P Lillicrap, Jonathan~J Hunt, Alexander Pritzel, Nicolas Heess, Tom
  Erez, Yuval Tassa, David Silver, and Daan Wierstra.
\newblock Continuous control with deep reinforcement learning.
\newblock In \emph{International Conference on Learning Representations}, 2016.

\bibitem[Luo et~al.(2019)Luo, Xu, Li, Tian, Darrell, and
  Ma]{luo2019algorithmic}
Yuping Luo, Huazhe Xu, Yuanzhi Li, Yuandong Tian, Trevor Darrell, and Tengyu
  Ma.
\newblock Algorithmic framework for model-based deep reinforcement learning
  with theoretical guarantees.
\newblock In \emph{International Conference on Learning Representations}, 2019.

\bibitem[Mao et~al.(2018{\natexlab{a}})Mao, Dong, and Lim]{mao2018universal}
Jiayuan Mao, Honghua Dong, and Joseph~J Lim.
\newblock Universal agent for disentangling environments and tasks.
\newblock In \emph{International Conference on Learning Representations},
  2018{\natexlab{a}}.

\bibitem[Mao et~al.(2018{\natexlab{b}})Mao, Li, Xie, Lau, Wang, and
  Smolley]{mao2018effectiveness}
Xudong Mao, Qing Li, Haoran Xie, Raymond Yiu~Keung Lau, Zhen Wang, and
  Stephen~Paul Smolley.
\newblock On the effectiveness of least squares generative adversarial
  networks.
\newblock \emph{IEEE transactions on pattern analysis and machine
  intelligence}, 2018{\natexlab{b}}.

\bibitem[Mnih et~al.(2015)Mnih, Kavukcuoglu, Silver, Rusu, Veness, Bellemare,
  Graves, Riedmiller, Fidjeland, Ostrovski, et~al.]{mnih2015human}
Volodymyr Mnih, Koray Kavukcuoglu, David Silver, Andrei~A Rusu, Joel Veness,
  Marc~G Bellemare, Alex Graves, Martin Riedmiller, Andreas~K Fidjeland, Georg
  Ostrovski, et~al.
\newblock Human-level control through deep reinforcement learning.
\newblock \emph{Nature}, 518\penalty0 (7540):\penalty0 529, 2015.

\bibitem[Munkres(1957)]{munkres1957algorithms}
James Munkres.
\newblock Algorithms for the assignment and transportation problems.
\newblock \emph{Journal of the society for industrial and applied mathematics},
  5\penalty0 (1):\penalty0 32--38, 1957.

\bibitem[Nachum et~al.(2018)Nachum, Gu, Lee, and Levine]{nachum2018data}
Ofir Nachum, Shixiang~Shane Gu, Honglak Lee, and Sergey Levine.
\newblock Data-efficient hierarchical reinforcement learning.
\newblock In \emph{Advances in Neural Information Processing Systems}, pages
  3303--3313, 2018.

\bibitem[Nair et~al.(2018)Nair, Pong, Dalal, Bahl, Lin, and
  Levine]{nair2018visual}
Ashvin~V Nair, Vitchyr Pong, Murtaza Dalal, Shikhar Bahl, Steven Lin, and
  Sergey Levine.
\newblock Visual reinforcement learning with imagined goals.
\newblock In \emph{Advances in Neural Information Processing Systems}, pages
  9191--9200, 2018.

\bibitem[Ng et~al.(1999)Ng, Harada, and Russell]{Ng1999}
Andrew~Y Ng, Daishi Harada, and Stuart~J Russell.
\newblock Policy invariance under reward transformations: Theory and
  application to reward shaping.
\newblock In \emph{Proceedings of the Sixteenth International Conference on
  Machine Learning}, pages 278--287. Morgan Kaufmann Publishers Inc., 1999.

\bibitem[Oh et~al.(2017)Oh, Singh, Lee, and Kohli]{oh2017zero}
Junhyuk Oh, Satinder Singh, Honglak Lee, and Pushmeet Kohli.
\newblock Zero-shot task generalization with multi-task deep reinforcement
  learning.
\newblock In \emph{International Conference on Machine Learning}, pages
  2661--2670, 2017.

\bibitem[Pathak et~al.(2018)Pathak, Mahmoudieh, Luo, Agrawal, Chen, Shentu,
  Shelhamer, Malik, Efros, and Darrell]{pathak2018zero}
Deepak Pathak, Parsa Mahmoudieh, Guanghao Luo, Pulkit Agrawal, Dian Chen, Yide
  Shentu, Evan Shelhamer, Jitendra Malik, Alexei~A Efros, and Trevor Darrell.
\newblock Zero-shot visual imitation.
\newblock In \emph{International Conference on Learning Representations}, 2018.

\bibitem[P{\'e}r{\'e} et~al.(2018)P{\'e}r{\'e}, Forestier, Sigaud, and
  Oudeyer]{pere2018unsupervised}
Alexandre P{\'e}r{\'e}, S{\'e}bastien Forestier, Olivier Sigaud, and
  Pierre-Yves Oudeyer.
\newblock Unsupervised learning of goal spaces for intrinsically motivated goal
  exploration.
\newblock In \emph{International Conference on Learning Representations}, 2018.

\bibitem[Plappert et~al.(2018)Plappert, Andrychowicz, Ray, McGrew, Baker,
  Powell, Schneider, Tobin, Chociej, Welinder, et~al.]{plappert2018multi}
Matthias Plappert, Marcin Andrychowicz, Alex Ray, Bob McGrew, Bowen Baker,
  Glenn Powell, Jonas Schneider, Josh Tobin, Maciek Chociej, Peter Welinder,
  et~al.
\newblock Multi-goal reinforcement learning: Challenging robotics environments
  and request for research.
\newblock \emph{arXiv preprint arXiv:1802.09464}, 2018.

\bibitem[Pong et~al.(2019)Pong, Dalal, Lin, Nair, Bahl, and
  Levine]{pong2019skew}
Vitchyr~H Pong, Murtaza Dalal, Steven Lin, Ashvin Nair, Shikhar Bahl, and
  Sergey Levine.
\newblock Skew-fit: State-covering self-supervised reinforcement learning.
\newblock \emph{arXiv preprint arXiv:1903.03698}, 2019.

\bibitem[Rauber et~al.(2019)Rauber, Ummadisingu, Mutz, and
  Schmidhuber]{rauber2019hindsight}
Paulo Rauber, Avinash Ummadisingu, Filipe Mutz, and J{\"u}rgen Schmidhuber.
\newblock Hindsight policy gradients.
\newblock In \emph{International Conference on Learning Representations}, 2019.

\bibitem[Riedmiller et~al.(2018)Riedmiller, Hafner, Lampe, Neunert, Degrave,
  Wiele, Mnih, Heess, and Springenberg]{riedmiller2018learning}
Martin Riedmiller, Roland Hafner, Thomas Lampe, Michael Neunert, Jonas Degrave,
  Tom Wiele, Vlad Mnih, Nicolas Heess, and Jost~Tobias Springenberg.
\newblock Learning by playing solving sparse reward tasks from scratch.
\newblock In \emph{International Conference on Machine Learning}, pages
  4341--4350, 2018.

\bibitem[Sahni et~al.(2019)Sahni, Buckley, Abbeel, and
  Kuzovkin]{sahni2019addressing}
Himanshu Sahni, Toby Buckley, Pieter Abbeel, and Ilya Kuzovkin.
\newblock Addressing sample complexity in visual tasks using her and
  hallucinatory gans.
\newblock In \emph{Advances in Neural Information Processing Systems}, 2019.

\bibitem[Schaul et~al.(2015)Schaul, Horgan, Gregor, and
  Silver]{schaul2015universal}
Tom Schaul, Daniel Horgan, Karol Gregor, and David Silver.
\newblock Universal value function approximators.
\newblock In \emph{International conference on machine learning}, pages
  1312--1320, 2015.

\bibitem[Schaul et~al.(2016)Schaul, Quan, Antonoglou, and
  Silver]{schaul2016prioritized}
Tom Schaul, John Quan, Ioannis Antonoglou, and David Silver.
\newblock Prioritized experience replay.
\newblock In \emph{International Conference on Learning Representations}, 2016.

\bibitem[Schulman et~al.(2015)Schulman, Levine, Abbeel, Jordan, and
  Moritz]{Schulman2015}
John Schulman, Sergey Levine, Pieter Abbeel, Michael Jordan, and Philipp
  Moritz.
\newblock Trust region policy optimization.
\newblock In \emph{International Conference on Machine Learning}, pages
  1889--1897, 2015.

\bibitem[Schulman et~al.(2017)Schulman, Wolski, Dhariwal, Radford, and
  Klimov]{Schulman2017ProximalPO}
John Schulman, Filip Wolski, Prafulla Dhariwal, Alec Radford, and Oleg Klimov.
\newblock Proximal policy optimization algorithms.
\newblock \emph{arXiv preprint arXiv:1707.06347}, 2017.

\bibitem[Silver et~al.(2016)Silver, Huang, Maddison, Guez, Sifre, Van
  Den~Driessche, Schrittwieser, Antonoglou, Panneershelvam, Lanctot,
  et~al.]{Silver16}
David Silver, Aja Huang, Chris~J Maddison, Arthur Guez, Laurent Sifre, George
  Van Den~Driessche, Julian Schrittwieser, Ioannis Antonoglou, Veda
  Panneershelvam, Marc Lanctot, et~al.
\newblock Mastering the game of go with deep neural networks and tree search.
\newblock \emph{nature}, 529\penalty0 (7587):\penalty0 484, 2016.

\bibitem[Srinivas et~al.(2018)Srinivas, Jabri, Abbeel, Levine, and
  Finn]{srinivas2018universal}
Aravind Srinivas, Allan Jabri, Pieter Abbeel, Sergey Levine, and Chelsea Finn.
\newblock Universal planning networks: Learning generalizable representations
  for visuomotor control.
\newblock In \emph{International Conference on Machine Learning}, pages
  4739--4748, 2018.

\bibitem[Sukhbaatar et~al.(2018)Sukhbaatar, Lin, Kostrikov, Synnaeve, Szlam,
  and Fergus]{sukhbaatar2018intrinsic}
Sainbayar Sukhbaatar, Zeming Lin, Ilya Kostrikov, Gabriel Synnaeve, Arthur
  Szlam, and Rob Fergus.
\newblock Intrinsic motivation and automatic curricula via asymmetric
  self-play.
\newblock In \emph{International Conference on Learning Representations}, 2018.

\bibitem[Sun et~al.(2019)Sun, Li, Liu, Lin, and Zhou]{sun2019policy}
Hao Sun, Zhizhong Li, Xiaotong Liu, Dahua Lin, and Bolei Zhou.
\newblock Policy continuation with hindsight inverse dynamics.
\newblock In \emph{Advances in Neural Information Processing Systems}, 2019.

\bibitem[Szepesv{\'a}ri(1998)]{Szepesvari1998}
Csaba Szepesv{\'a}ri.
\newblock The asymptotic convergence-rate of q-learning.
\newblock In \emph{Advances in Neural Information Processing Systems}, pages
  1064--1070, 1998.

\bibitem[Thomas et~al.(2017)Thomas, Pondard, Bengio, Sarfati, Beaudoin, Meurs,
  Pineau, Precup, and Bengio]{thomas2017independently}
Valentin Thomas, Jules Pondard, Emmanuel Bengio, Marc Sarfati, Philippe
  Beaudoin, Marie-Jean Meurs, Joelle Pineau, Doina Precup, and Yoshua Bengio.
\newblock Independently controllable features.
\newblock \emph{arXiv preprint arXiv:1708.01289}, 2017.

\bibitem[Veeriah et~al.(2018)Veeriah, Oh, and Singh]{veeriah2018many}
Vivek Veeriah, Junhyuk Oh, and Satinder Singh.
\newblock Many-goals reinforcement learning.
\newblock \emph{arXiv preprint arXiv:1806.09605}, 2018.

\bibitem[Zhao and Tresp(2018)]{zhao2018energy}
Rui Zhao and Volker Tresp.
\newblock Energy-based hindsight experience prioritization.
\newblock In \emph{Conference on Robot Learning}, pages 113--122, 2018.

\bibitem[Zhao et~al.(2019)Zhao, Sun, and Tresp]{zhao2019maximum}
Rui Zhao, Xudong Sun, and Volker Tresp.
\newblock Maximum entropy-regularized multi-goal reinforcement learning.
\newblock In \emph{International Conference on Machine Learning}, pages
  7553--7562, 2019.

\end{thebibliography}

\clearpage
\appendix
\section{Proof of Theorem 1}
\label{app:proof}
In this section we provide the proof of Theorem 1.
\thmbound*

    \begin{proof}
    By Eq.~\eqref{discrepancy_bound}, for any quadruple $(s, g, s', g')$, we have
    \begin{align}
    V^{\pi}(s', g') \geq V^{\pi}(s, g) - L \cdot d((s, g), (s', g')). \label{eq:proof-theorem-1}
    \end{align}
    For any $\mu \in\Gamma(\mathcal{T},\mathcal{T}')$, we sample $(s, g, s', g') \sim \mu$ and take the expectation on both sides of Eq.~\eqref{eq:proof-theorem-1}, and get
    \begin{align}
        V^{\pi}(\mathcal{T}') \geq V^{\pi}(\mathcal{T}) - L\cdot \mathbb{E}_\mu [d((s, g), (s', g'))] . \label{eq:proof-theorem-2}
    \end{align}
    Since Eq.~\eqref{eq:proof-theorem-2} holds for any $\mu \in \Gamma(\mathcal{T}, \mathcal{T'})$, we have
    \[
      V^{\pi}(\mathcal{T}') \geq V^{\pi}(\mathcal{T}) - L\cdot \inf_{\mu \in \Gamma(\mathcal{T}, \mathcal{T}')} \left(\mathbb{E}_\mu [d((s, g), (s', g'))] \right) = V^{\pi}(\mathcal{T}) - L\cdot D(\mathcal{T}, \mathcal{T'}) .
    \]
    \end{proof}

\section{Experiment Settings}\label{experiment_settings}

\subsection{Modified Environments}\label{modified_env}
    \begin{figure}[h]
        \centering
        \includegraphics[width=0.30\textwidth]{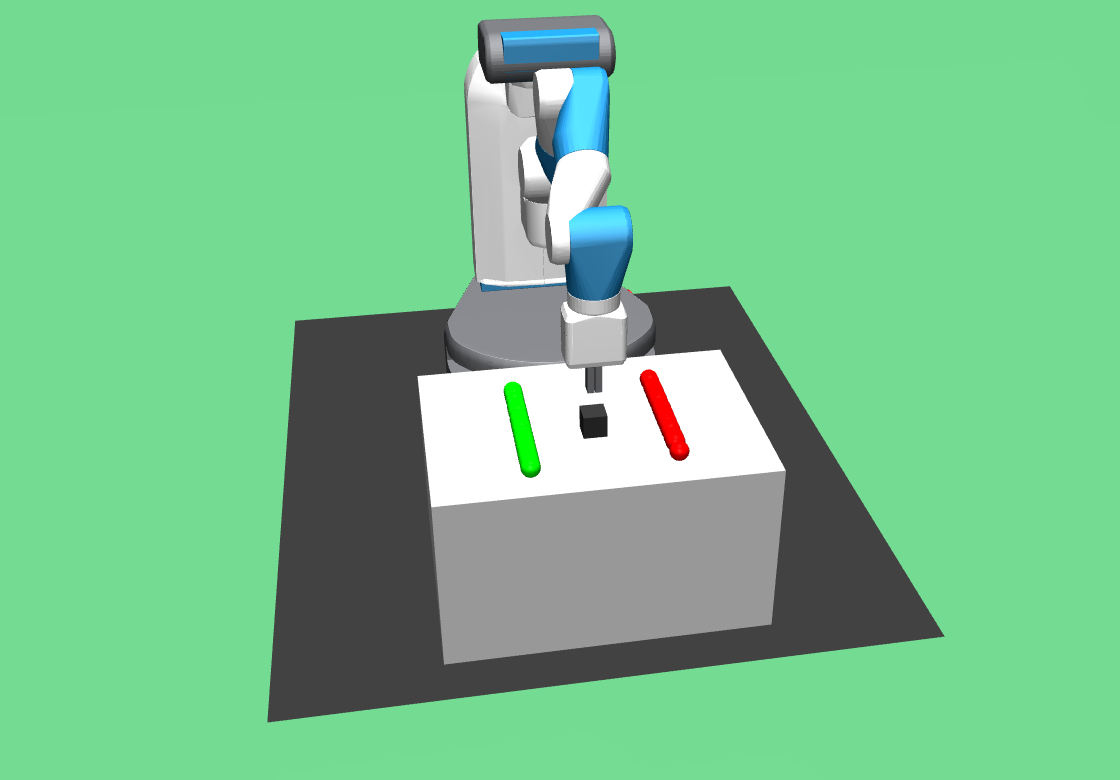}
        \includegraphics[width=0.30\textwidth]{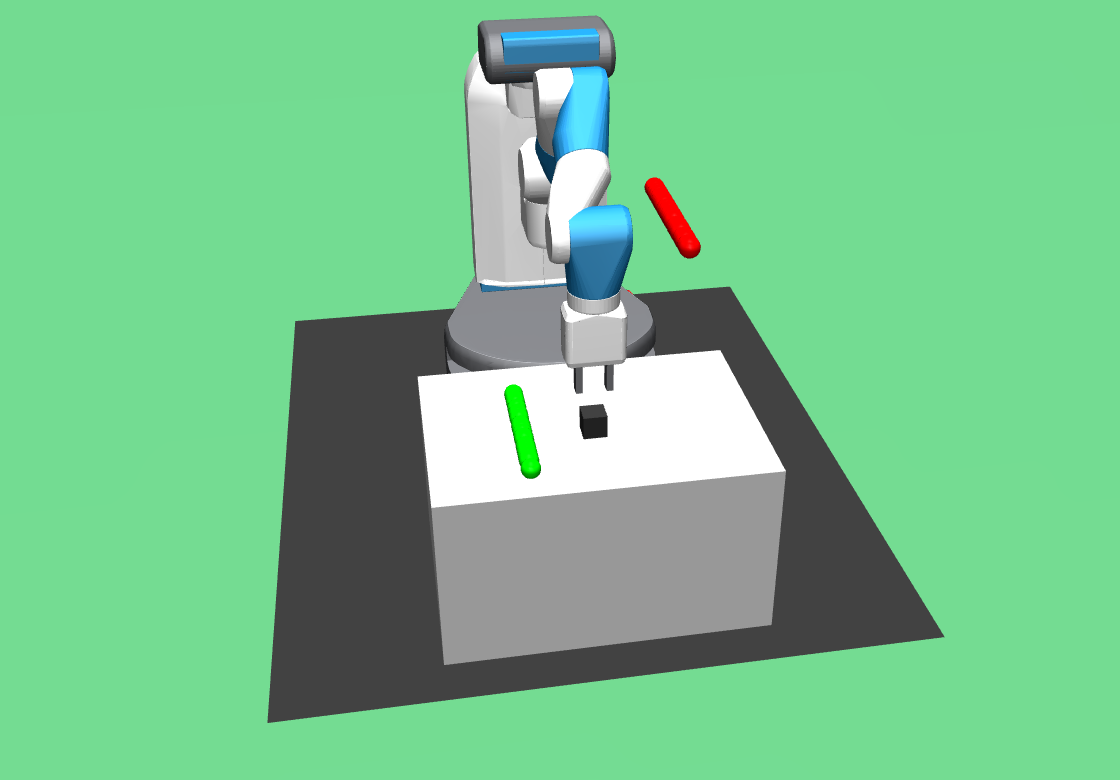}
        \includegraphics[width=0.30\textwidth]{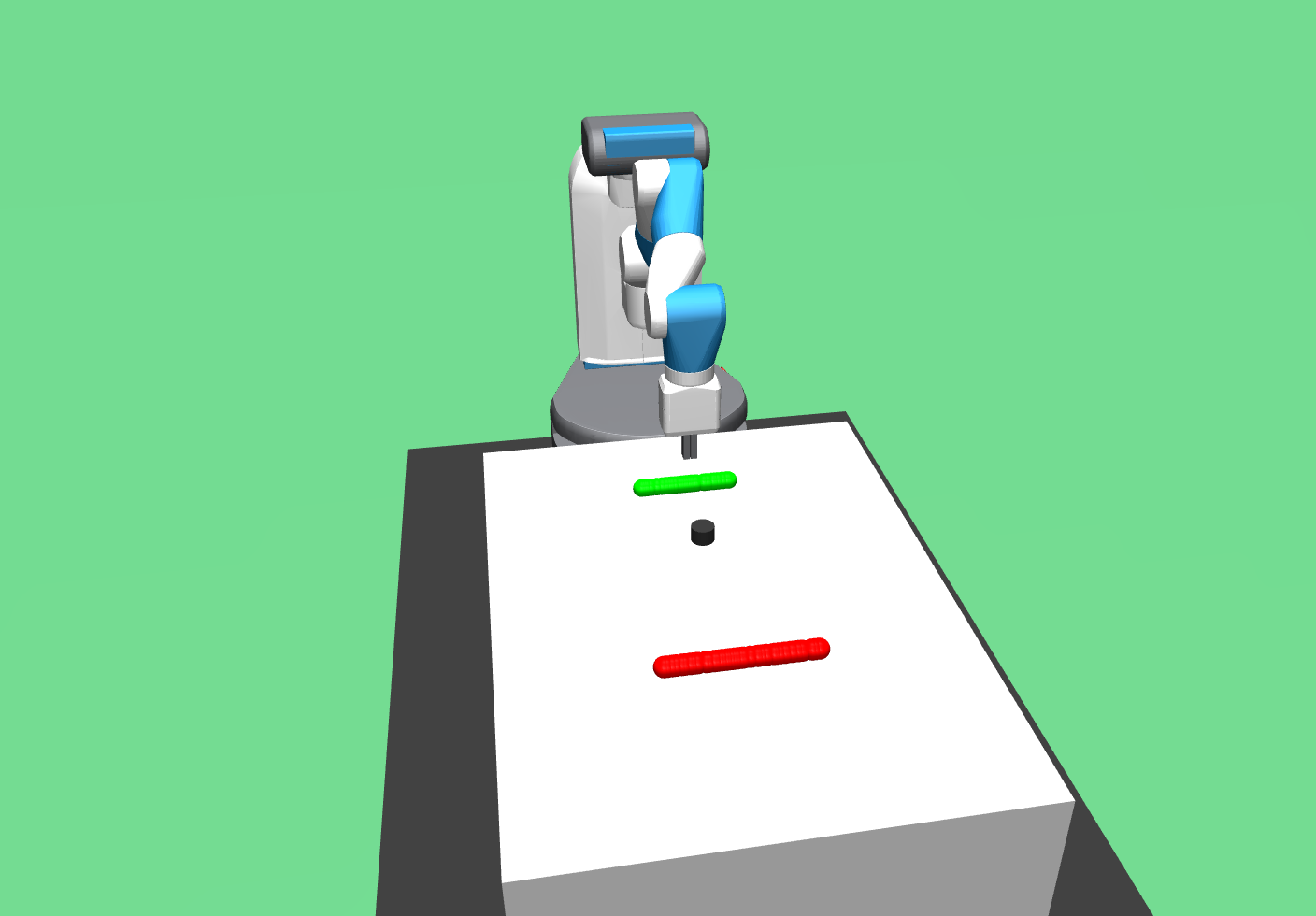}
        \caption{Visualization of modified task distribution in Fetch environments. The object is uniformly generated on the green segment, and the goal is uniformly generated on the red segment.}
    \end{figure}
    
    Fetch Environments:  
    \begin{itemize}
        \item FetchPush-v1: Let the origin $(0,0,0)$ denote the projection of gripper's initial coordinate on the table. The object is uniformly generated on the segment $(-0.15,-0.15,0)-(0.15,-0.15,0)$, and the goal is uniformly generated on the segment $(-0.15,0.15,0)-(0.15,0.15,0)$.
        \item FetchPickAndPlace-v1: Let the origin $(0,0,0)$ denote the projection of gripper's initial coordinate on the table. The object is uniformly generated on the segment $(-0.15,-0.15,0)-(0.15,-0.15,0)$, and the goal is uniformly generated on the segment $(-0.15,0.15,0.45)-(0.15,0.15,0.45)$.
        \item FetchSlide-v1: Let the origin $(0,0,0)$ denote the projection of gripper's initial coordinate on the table. The object is uniformly generated on the segment $(-0.05,-0.1,0)-(-0.05,0.1,0)$, and the goal is uniformly generated on the segment $(0.55,-0.15,0)-(0.55,0.15,0)$.
    \end{itemize}
    
    Hand Environments: 
    \begin{itemize}
        \item HandManipulateBlockRotate-v0, HandManipulateEggRotate-v0: Let $s_0$ be the default initial state defined in original simulator \citep{plappert2018multi}. The initial pose is generated by applying a rotation around $z$-axis, where the rotation degree will be uniformly sampled from $[-\pi/4,\pi/4]$. The goal is also rotated from $s_0$ around $z$-axis, where the degree is uniformly sampled from $[\pi-\pi/4,\pi+\pi/4]$.
        \item HandManipulatePenRotate-v0: We use the same setting as the original simulator.
    \end{itemize}
    
    Reach Environments:
    \begin{itemize}
        \item FetchReach-v1: Let the origin $(0,0,0)$ denote the coordinate of gripper's initial position. Goal is uniformly generated on the segment $(-0.15,0.15,0.15)-(0.15,0.15,0.15)$.
        \item HandReach-v0: Uniformly select one dimension of \emph{meeting point} and add an offset of 0.005, where \emph{meeting point} is defined in original simulator  \citep{plappert2018multi}
    \end{itemize}
    
    Other attributes of the environment (such as horizon $H$, reward function $R_g$) are kept the same as default.

\subsection{Evaluation Details}
    \begin{itemize}
        \item All curves presented in this paper are plotted from 10 runs with random task initializations and seeds.
        \item Shaded region indicates 60\% population around median.
        \item All curves are plotted using the same hyper-parameters (except ablation section).
        \item Following \citet{andrychowicz2017hindsight}, an episode is considered successful if $\|\phi(s_H)-g\|_2\leq \delta_g$ is achieved, where $\phi(s_H)$ is the object position at the end of the episode. $\delta_g$ is the same threshold using in reward function \eqref{reward_setting}.
    \end{itemize}

\subsection{Details of Experiment with obstacle} \label{grid_metric_details}
    Using the same coordinate system as Appendix \ref{modified_env}. Let the origin $(0,0,0)$ denote the projection of gripper's initial coordinate on the table. The object is uniformly generated on the segment $(-0.15,-0.15,0)-(-0.045,-0.15,0)$, and the goal is uniformly generated on the segment $(-0.15,0.15,0)-(-0.045,0.15,0)$. The wall lies on $(-0.3,0,0)-(0,0,0)$.
    
    The crafted distance used in Figure \ref{obstacle_task} is calculated by the following rules.
    \begin{itemize}
        \item The distance metric between two initial states is kept as before.
        \item The distance between the hindsight goal $g$ and the desired goal $g^*$ is evaluated as the summation of two parts. The first part is the $\ell_2$ distance between the goal $g$ and its closest point $g'$ on the blue polygonal line shown in Figure \ref{obstacle_task}. The second part the distance between $g'$ and $g^*$ along the blue line.
        \item The above two terms are comined with the same ratio used in Eq.~\eqref{state_abstraction}.
    \end{itemize}

\subsection{Details of Experiment \ref{compare_curriculum}}
    \begin{figure}[h]
        \centering
        \includegraphics[width=0.45\textwidth]{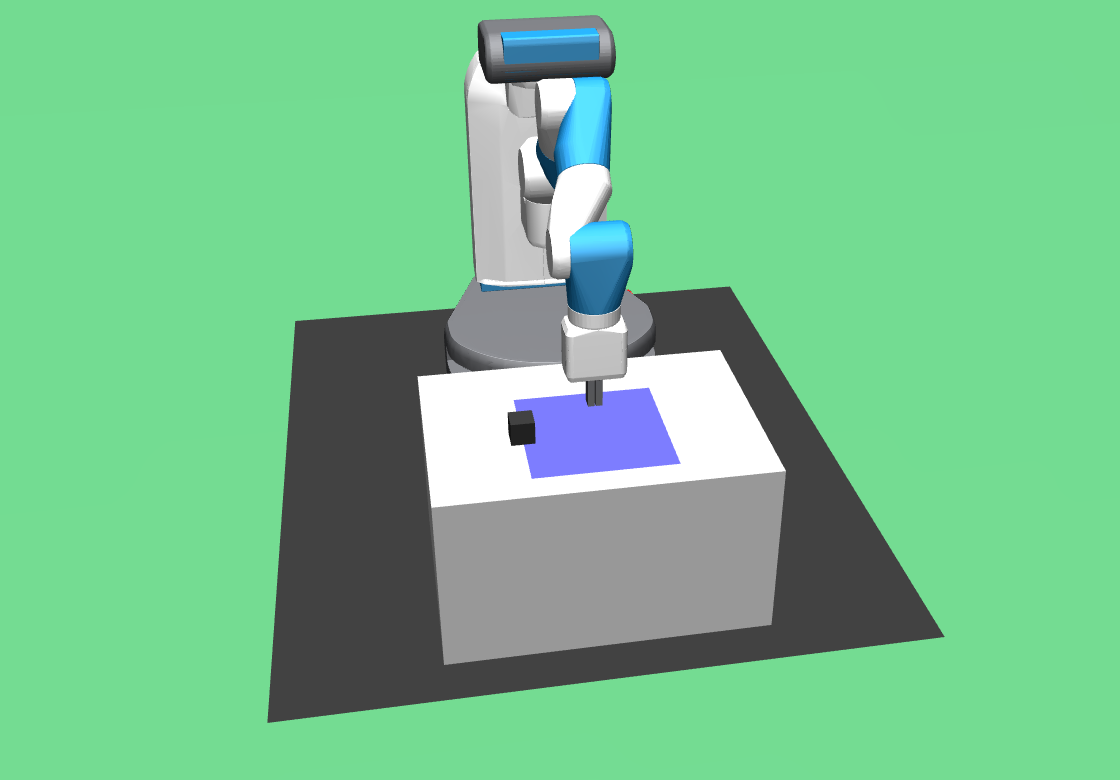}
        \caption{Visualization of modified task distribution in Experiment \ref{compare_curriculum}. The initial position of the object is as shown in this figure, and the goal is uniformly generated in the blue region.}
    \end{figure}
    
    \begin{itemize}
        \item Since the environment is deterministic, the success rate $f(\pi,g)$ is defines as
        $$f(\pi,g)=\int_{g'\in \mathcal{B}(g,\delta_g)}\mathbbm{1}[\pi \text{ achieves success for the goal } g']\;d g',$$ where $\mathcal{B}(g,\delta_g)$ indicates a ball with radius $\delta_g$, centered at $g$. And $\delta_g$ is the same threshold using in reward function \eqref{reward_setting} and success testing.
        \item 
        The average success rate oracle $f(\pi,g)$ is estimated by $10^2$ samples.
    \end{itemize}

\section{Implementation Details}
\subsection{Hyper-Parameters}\label{hyper_parameters}
    Almost all hyper-parameters using DDPG and HER are kept the same as benchmark results, only following terms differ with \citet{plappert2018multi}:
    \begin{itemize}
        \item number of MPI workers: 1;
        \item buffer size: $10^4$ trajectories.
    \end{itemize}
    
    Other hyper-parameters:
    \begin{itemize}
        \item Actor and critic networks: 3 layers with 256 units and ReLU activation;
        \item Adam optimizer with $10^{-3}$ learning rate;
        \item Polyak-averaging coefficient: 0.95;
        \item Action $L_2$-norm penalty coefficient: 1.0;
        \item Batch size: 256;
        \item Probability of random actions: 0.3;
        \item Scale of additive Gaussian noise: 0.2;
        \item Probability of HER experience replay: 0.8;
        \item Number of batches to replay after collecting one trajectory: 20.
    \end{itemize}
    
    Hyper-parameters in weighted bipartite matching:
    \begin{itemize}
        \item Lipschitz constant $L$: 5.0;
        \item Distance weight $c$: 3.0;
        \item Number of hindsight goals $K$: 50 or 100.
    \end{itemize}
    
\subsection{Details on Data Processing}
\begin{itemize}
    \item In policy training of HGG, we sample minibatches using HER.
    \item As a normalization step, we use Lipschitz constant $L^*=\frac{L}{(1-\gamma)d^{max}}$ in back-end computation, where $d^{max}$ is the $\ell_2$-diameter of the goal space $\mathcal{G}$, and $L$ corresponds to the amount discussed in ablation study.
    \item To reduce computational cost of bipartite matching, we approximate the buffer set by a First-In-First-Out queue containing $10^3$ recent trajectories.
    \item An additional Gaussian noise $\mathcal{N}(0,0.05I)$ is added to goals generated by HGG in Fetch environments. We don't add this term in Hand environments because the goal space is not $\mathbb{R}^d$.
\end{itemize}

\section{Additional Experiment Results} \label{appendix_experiment}

\subsection{Additional Visualization of Hindsight Goals Generated by HGG}

\begin{figure}[h]
    \centering
    \begin{subfigure}{0.30\textwidth}
        \begin{minipage}{0.98\textwidth}
            \centering
            \includegraphics[height=3.1cm]{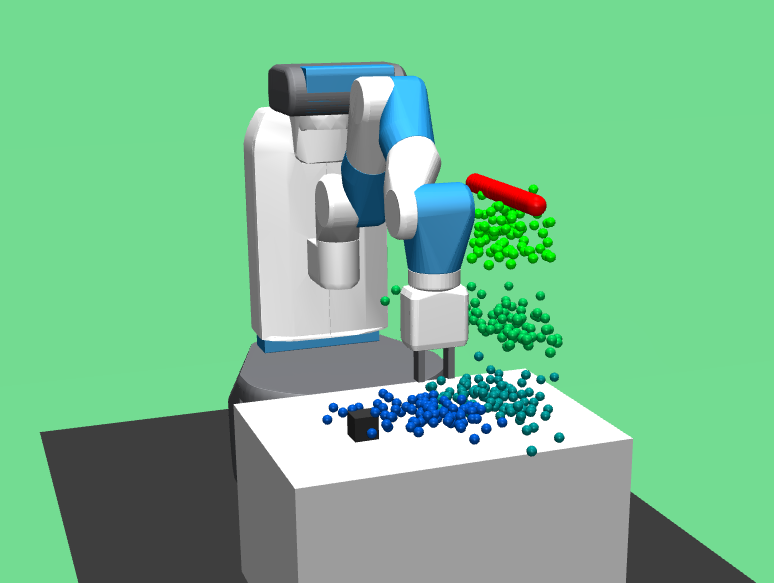}
            \vspace{0.00in}\caption{}
            \label{additional_visualization_1}
        \end{minipage}
    \end{subfigure}
    \begin{subfigure}{0.33\textwidth}
        \begin{minipage}{0.98\textwidth}
            \centering
            \includegraphics[height=3.5cm]{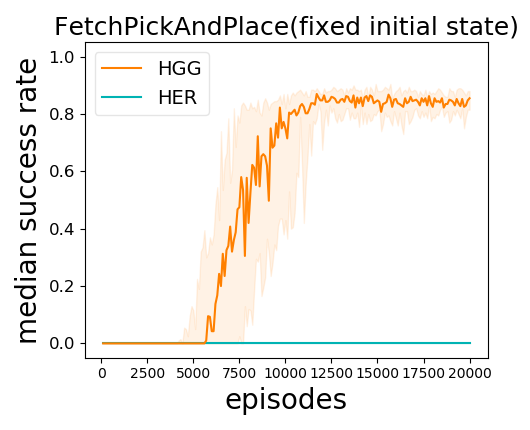}
            \caption{}
            \label{additional_visualization_2}
        \end{minipage}
    \end{subfigure}
    \begin{subfigure}{0.33\textwidth}
        \begin{minipage}{0.98\textwidth}
            \centering
            \includegraphics[height=3.5cm]{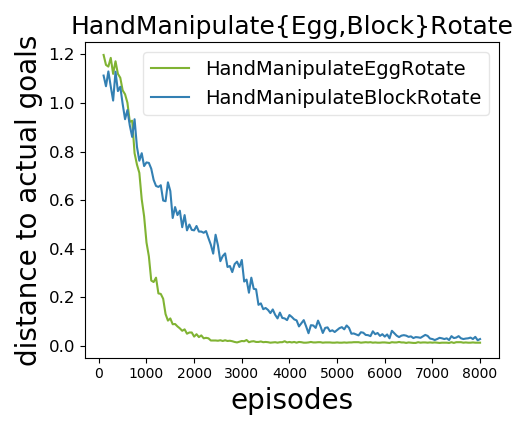}
            \caption{}
            \label{additional_visualization_3}
        \end{minipage}
    \end{subfigure}
    \caption{Additional visualization to illustrate the hindsight goals generated by HGG.}
\end{figure}

To give better intuitive illustrations on our motivation, we provide an additional visualization of goal distribution generated by HGG on a complex manipulation task FetchPickAndPlace (Figures \ref{additional_visualization_1} and \ref{additional_visualization_2}). In Figure \ref{additional_visualization_1}, ``blue to green'' corresponds to the generated goals during training. HGG will guide the agent to understand the location of the object in the early stage, and move it to its nearby region. Then it will learn to move the object towards the easiest direction, i.e. pushing the object to the location underneath the actual goal, and finally pick it up. For those tasks which are hard to visualize, such as the HandManipultation tasks, we plotted the curves of distances between proposed exploratory goals and actually desired goals (Figure \ref{additional_visualization_3}), all experiment followed the similar learning dynamics.

\subsection{Evaluation on Standard Tasks}
    In this section, we provide experiment results on standard Fetch tasks. The learning are shown in Figure~\ref{fig:appd2}.
    \begin{figure}[h]
        \centering
        \includegraphics[width=0.99\textwidth]{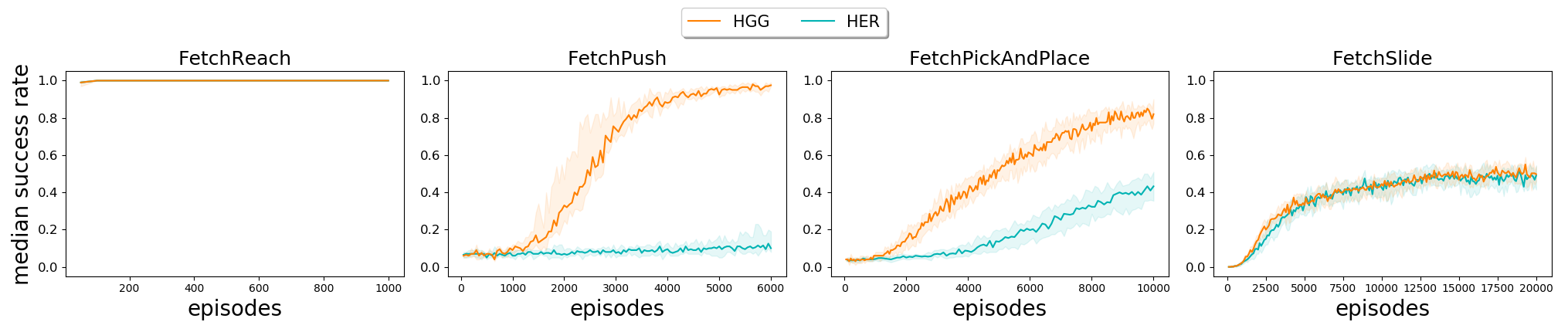}
        \caption{Learning curves for HGG and HER in standard task distribution created by \citet{andrychowicz2017hindsight}.}
        \label{fig:appd2}
    \end{figure}

\subsection{Additional Experiment Results on Section \ref{compare_curriculum}}
We provide the comparison of the performance of HGG and explicit curriculum learning on FetchPickAndPlace environment (see Figure \ref{fig:appd3}), showing that the result given in Section \ref{compare_curriculum} generalizes to a different environment.
    \begin{figure}[H]
        \centering
        \begin{minipage}{0.38\textwidth}
            \centering\vspace{-0.15in}
            \includegraphics[height=3.5cm]{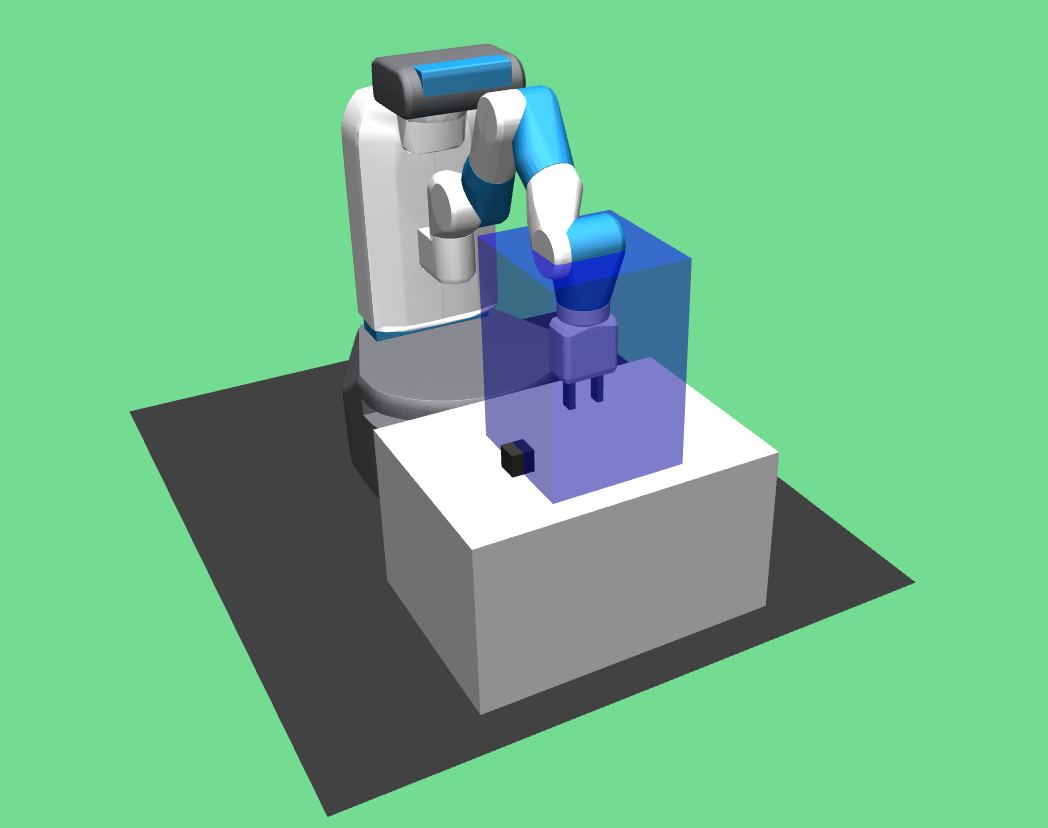}
        \end{minipage}
        \begin{minipage}{0.38\textwidth}
            \centering
            \includegraphics[height=4cm]{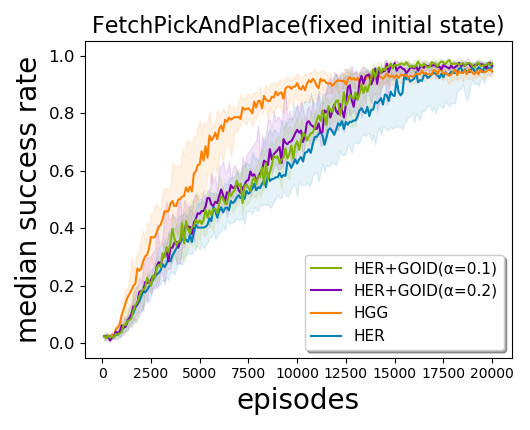}
        \end{minipage}
        \caption{Comparison with explicit curriculum learning in FetchPickAndPlace. The initial position of the object is as shown in the left figure, and the goal is generated in the blue region following the default distribution created by \citet{andrychowicz2017hindsight}.}
        \label{fig:appd3}
    \end{figure}

\subsection{Ablation Study}\label{appendix_ablation}
We provide full experiments on ablation study in Figure~\ref{fig:appd4}.
    \begin{figure}[H]
        \centering
        \includegraphics[width=0.95\textwidth]{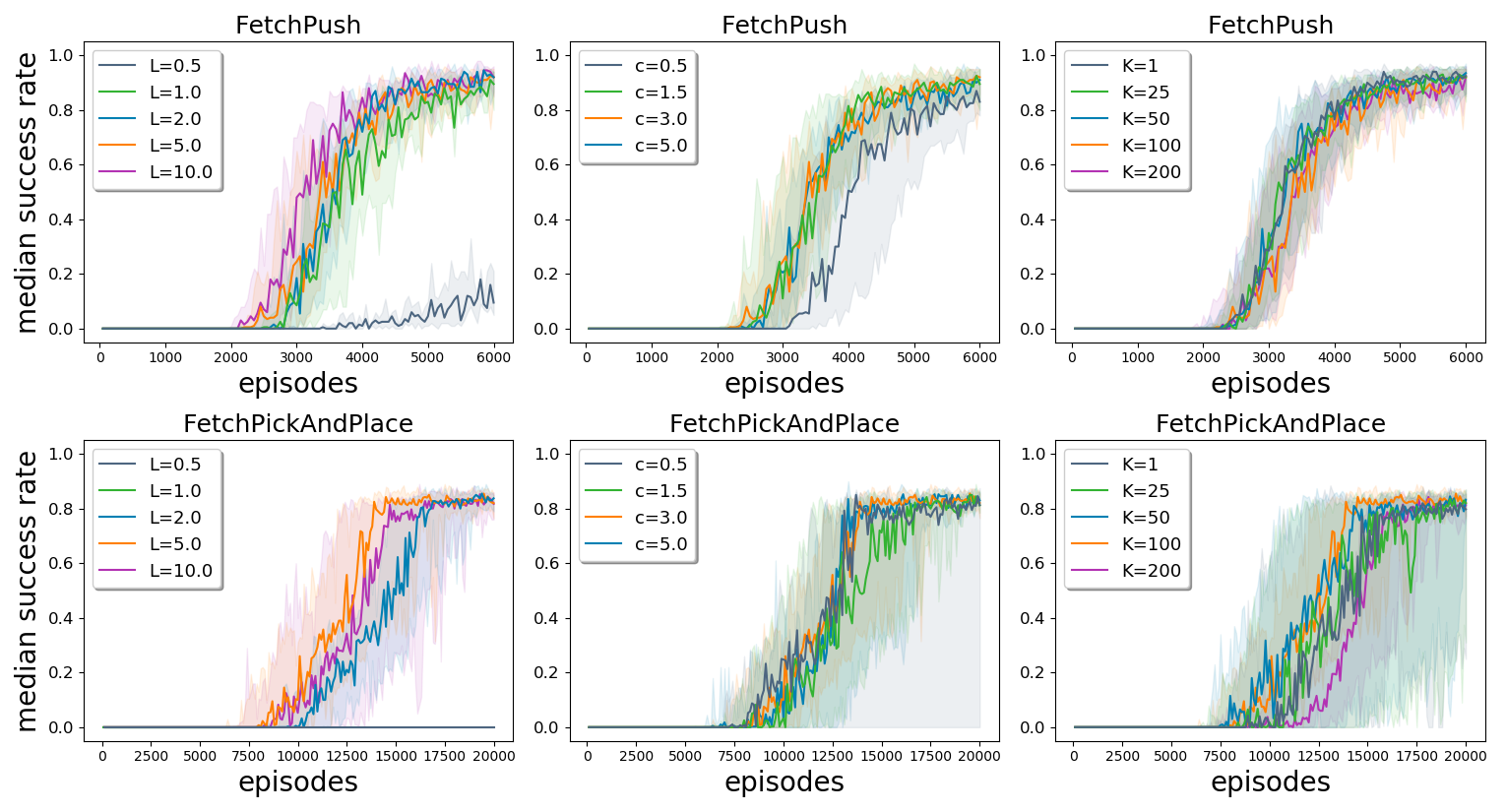}
        \caption{A full version of ablation study.}
        \label{fig:appd4}
    \end{figure}

\end{document}